\documentclass[11pt,oneside,a4paper]{article}
\usepackage[top=1 in, bottom=1 in, left=0.85 in, right=0.85 in]{geometry}
\usepackage{hyperref}
 \usepackage{float}
\usepackage{subcaption}

\usepackage{amsmath,amssymb,graphicx,url, algorithm2e}
\usepackage{thmtools,thm-restate,wrapfig,enumitem,mathabx}
\newtheorem{assumption}{Assumption}
\newtheorem{theorem}{Theorem}

\newtheorem{lemma}{Lemma}
\newtheorem{corollary}{Corollary}
\newtheorem{proof}{Proof}

\DeclareMathOperator*{\argmax}{arg\,max}

\title{An Efficient Pessimistic-Optimistic Algorithm for Stochastic Linear\\ Bandits with General Constraints\footnote{A preliminary version of this paper that considers the traditional multi-armed bandits can be found in \cite{LiuLiShi_20}.}}
\author{Xin Liu \\University of Michigan, Ann Arbor \\ xinliuee@umich.edu
   \and Bin Li \\Pennsylvania State University \\ binli@psu.edu
   \and ~~~~~~~
   \and Pengyi Shi \\~~~~Purdue University~~~~ \\ shi178@purdue.edu
   \and Lei Ying \\University of Michigan, Ann Arbor \\ leiying@umich.edu}
\date{}

\begin{document}

\maketitle

\begin{abstract}
This paper considers stochastic linear bandits with general nonlinear constraints. 
The objective is to maximize the expected cumulative reward over horizon $T$ subject to a set of constraints in each round $\tau\leq T$. We propose a pessimistic-optimistic algorithm for this problem, which is efficient in two aspects. First, the algorithm yields $\tilde{\cal O}\left(\left(\frac{K^{0.75}}{\delta}+d\right)\sqrt{\tau}\right)$ (pseudo) regret in round $\tau\leq T,$ where $K$ is the number of constraints, $d$ is the dimension of the reward feature space, and $\delta$ is a Slater's constant; and  {\em zero}  constraint violation in any round $\tau>\tau',$ where $\tau'$ is  {\em independent} of horizon $T.$ Second, the algorithm is computationally efficient. 
Our algorithm is based on the primal-dual approach in optimization and includes two components. The primal component is similar to unconstrained stochastic linear bandits (our algorithm uses the linear upper confidence bound algorithm (LinUCB)). The computational complexity of the dual component depends on the number of constraints, but is independent of the sizes of the contextual space, the action space, and the feature space. Thus, the overall computational complexity of our algorithm is similar to that of the linear UCB for unconstrained stochastic linear bandits.  
\end{abstract}

\section{Introduction}
Stochastic linear bandits have a broad range of applications in practice, including online recommendations, job assignments in crowdsourcing, and clinical trials in healthcare. Most existing studies on stochastic linear bandits formulated them as unconstrained online optimization problems, limiting their application to problems with operational constraints such as safety, fairness, and budget constraints. 
In this paper, we consider a stochastic linear bandit with general constraints. 
As in a standard stochastic linear bandit, at the beginning of each round $t\in [T],$ the learner is given a context $c(t)$ that is randomly sampled from the context set $\mathcal C$ (a countable set), and takes an action $A(t)\in[J].$ The learner then receives a reward $R(c(t), A(t))=r(c(t),A(t)) + \eta(t),$ where $r(c, j) = \langle \theta_{*}, \phi(c, j) \rangle,$ $\phi(c,j)\in \mathbb R^d$ is a $d$-dimensional feature vector for (context, action) pair $(c,j),$ $\theta_{*} \in \mathbb R^d$ is an unknown underlying vector to be learned, and $\eta(t)$ is a zero-mean random variable. 
For constrained stochastic linear bandits, we further assume when action $A(t)$ is taken on context $c(t),$ it incurs $K$ different types of costs, denoted by $W^{(k)}(c(t), A(t)).$ We assume $W^{(k)}(c,j)$ is a random variable with mean $w^{(k)}(c,j)$ that is unknown to the learner.  
This paper considers general cost functions and does {\em not} require $w^{(k)}(c,j)$ to have a linear form like $r(c,j)$.

Denote the action taken by policy $\pi$ in round $t$ by $A^{\pi}(t)$. The learner's objective is to learn a policy $\pi$ that maximizes the cumulative rewards over horizon $T$ subject to {\em anytime cumulative constraints}: 
\begin{align}
    &\max_{\pi} \mathbb E\left[\sum_{t=1}^{T} R(c(t), A^\pi(t)) \right] 
    \label{obj-intro}\\
\hbox{subject to: }& \mathbb E\left[\sum_{t=1}^{\tau} W^{(k)}\left(c(t), A^\pi(t)\right) \right] \leq 
0, ~\forall\ \tau\in[T], k\in[K]. \label{eq:cons-intro}
\end{align}
The constraint \eqref{eq:cons-intro} above may represent different operational constraints including safety, fairness, and budget constraints.

\subsection*{Anytime cumulative constraints}
In the literature, 
constraints in stochastic bandits have been formulated differently. There are two popular formulations. The first one is a cumulative constraint over horizon $T,$ including knapsack bandits \cite{BadLanSli_14,BadAshKle_18,AgrDev_16,AgrDev_14,AgrDevLi_16,FerSimWan_18,CayErySri_20} where the process terminates when the total budget has been consumed; fair bandits where the number of times an action can be taken must exceed a threshold at the end of the horizon \cite{CheCueLuo_20}; and contextual bandits with a cumulative budget constraint \cite{WuSriLiu_15,ComJiaSri_15}. In these settings, the feasible action set in each round depends on the history. In general, the learner has more flexibility in the earlier rounds, close to that in  the unconstrained setting. Another formulation is {\em anytime} constraints, which either require the expected cost of the action taken in each round to be lower than a threshold  \cite{AmaAliThr_19,MorAmaAli_19} or the expected cost of the policy in each round is lower than a threshold \cite{PacGhaBar_20}. We call them {\em anytime} action constraints and {\em anytime} policy constraints, respectively. 

\begin{wrapfigure}{r}{0.45\textwidth}
\centering
  \includegraphics[width=0.4\textwidth]{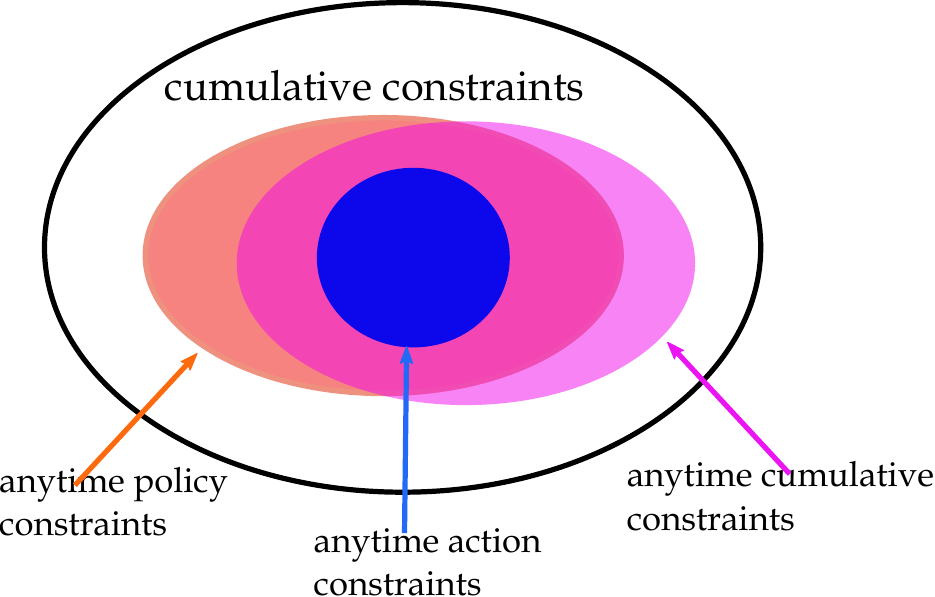}
  \caption{\small{A conceptual description of feasible policy sets under different constraint formulations.}}
  \label{fig:constraints}
\end{wrapfigure}

Our constraint in the form of \eqref{eq:cons-intro} is an {\em anytime cumulative constraint}, i.e., it imposes a cumulative constraint in {\em every} round. This anytime cumulative constraint is most similar to the anytime policy constraint in \cite{PacGhaBar_20} because the average cost of a policy is close to its mean after the policy has been applied for many rounds and the process converges, so it can be viewed as a cumulative constraint on actions over many rounds (like ours). Furthermore, when our anytime cumulative constraint  \eqref{eq:cons-intro} is satisfied, it is guaranteed that the time-average cost is below a threshold in every round. 

In summary, our anytime cumulative constraint is stricter than a cumulative constraint over fixed horizon $T$ but is less restrictive than anytime action constraint in \cite{AmaAliThr_19,MorAmaAli_19}. Figure~\ref{fig:constraints} provides a conceptual description of the relationship between these different forms of constraints.  

\subsection*{Main Contributions}
\label{sec:contribution}

This paper presents a pessimistic-optimistic algorithm based on the primal-dual approach in optimization for the problem defined in \eqref{obj-intro}-\eqref{eq:cons-intro}.  The algorithm is efficient in two aspects. First, the algorithm yields $\tilde{\cal O}\left(\left(\frac{K^{0.75}}{\delta}+d\right)\sqrt{\tau}\right)$  regret in round $\tau\leq T$ and achieves {\em zero}  constraint violation in any round $\tau>\tau'$ for a constant $\tau'$  {independent} of horizon $T$. Second, the algorithm is computationally efficient.

For computational efficiency, the design of our algorithm is based on the primal-dual approach in optimization. The computation of the primary component is similar to unconstrained stochastic linear bandits \cite{DanHayTho_08,RusTsi_10,LiChuLan_10,AbbPalSze_11,ChuLiRey_11}. 
The dual component includes a set of Lagrangian multipliers that are updated in a simple manner to keep track of the levels of constraint violations so far in each round; the update depends on the number of constraints, but it is independent of the sizes of the contextual space, the action space, and the feature space. Thus, the overall computational complexity of our algorithm is similar to that of LinUCB in the unconstrained setting. This results in a much more efficient calculation comparing to OPLB proposed in \cite{PacGhaBar_20}. OPLB needs to construct a safe policy set in each round, hence, its computational complexity is prohibitively high as the authors acknowledged. 

For constraint violation, our algorithm guarantees that for any $\tau>\tau',$ the constraint holds with probability one. In other words, after a constant number of rounds, the constraint is always satisfied. This is in contrast to prior works \cite{PacGhaBar_20,AmaAliThr_19}, where anytime constraints are proven to hold over horizon $T$ with probability $1-\chi$ for a constant $\chi$. In other words, the anytime constraints may be violated with probability $\chi,$ and it is not clear how often they are violated when it happens. Furthermore, beyond mean cost constraints considered in \eqref{eq:cons-intro} and in  \cite{PacGhaBar_20,AmaAliThr_19}, we prove that a sample-path version of constraint \eqref{eq:cons-intro} holds with probability $1-O\left(e^{-\frac{\delta\sqrt{\tau}}{50K^{2.5}}}\right)$ in round $\tau$ under our algorithm. 

To summarize, our algorithm is computationally efficient and provides strong guarantees on both regret and constraint violations. Additionally, our cost function is in a general form and does not need to be linear as those in \cite{PacGhaBar_20,AmaAliThr_19}. 
We discuss more related work in the following.

\subsection*{Related Work}
\label{sec:related}
Stochastic linear bandits \cite{AbeLon_99, Auer_03} are a special class of contextual bandits \cite{Mic_79, LanZha_08}, which generalize multi-armed bandits \cite{LaiRob_85}. 
Besides \cite{PacGhaBar_20}, \cite{CheCueLuo_20} considered an adversarial contextual bandit with anytime policy constraint representing fairness. 
The proposed algorithm has $\tilde{\cal O}(\sqrt{|\mathcal C|JT})$ regret when the context distribution is known to the learner; otherwise it has $\tilde{\cal O}(\sqrt{|\mathcal C|JT})$ regret and $\tilde{\cal O}(\sqrt{|\mathcal C|T})$ constraint violation. \cite{LiLiuJi_19} studied a combinatorial sleeping bandits problem under cumulative fairness
constraints and proposed an algorithm based on UCB which they {\em conjectured} to have $\tilde{\cal O}(\sqrt{T})$ regret and $\tilde{\cal O}(\sqrt{T})$ constraint violation. 
Recent work studied unconstrained structured bandits and proposed primal-dual approach based on asymptotically lower bound problem in bandits \cite{KirLatVerSze_21, TirPirResLaz_20, DegShaKoo_20}.
However, our algorithm is different from them in three aspects. Our primal component is a greedy algorithm instead of a (sub-)gradient algorithm (as in \cite{TirPirResLaz_20}). Our dual component does not solve a best response problem, which is a constrained optimization problem as in \cite{KirLatVerSze_21, DegShaKoo_20}. Our analysis is based on the Lyapunov-drift analysis for queueing systems, e.g., we establish a bound on the exponential moment of the dual variable, which is not present in \cite{KirLatVerSze_21, TirPirResLaz_20, DegShaKoo_20}. It is also worth mentioning that \cite{KiaEil_20, AbbMohYad_17,AhmChrMah_20} studied ``conservative'' bandits which require that the reward or the cumulative reward exceeds a threshold at each step. Another line of related work is
online convex optimization with constraints, studied in \cite{ShiJohYu_09,MahJinYan_12,YuNee_20,YuNeeWei_17,WeiYuNee_20,UsmKraKam_19}, where online primal-dual with proximal regularized algorithms have been proposed to achieve $O(\sqrt{T})$ regret and $O(1)$ violation for static constraints in \cite{YuNee_20} and $O(\sqrt{T})$ violation for stochastic constraints in \cite{YuNeeWei_17} .

\noindent{\bf Notation.} 
$f(n) = \tilde{\mathcal O}(g(n))$ denotes $f(n) = O(g(n){\log}^k n)$ with $k>0;$ $[N]$ denotes the set $\{1,2,\cdots, N\};$ $\langle \cdot, \cdot \rangle$ denotes the inner product; $(\cdot)^\dag$ denotes the transpose of a vector or a matrix; $||\cdot||=||\cdot||_2,$ and $||\mathbf x||_{\Sigma} = \sqrt{\mathbf x^{\dag} \Sigma \mathbf x}.$ We add subscript $t$ to a variable when it is a time-varying sequence of constants (e.g., $c_t$), and add $(t)$ when they are random variables or decision variables (e.g., $c(t)$). We summarize our notation in Appendix \ref{app:notation}.

\section{A Pessimistic-Optimistic Algorithm}
\label{sec:alg}

We consider a stochastic linear bandit over  horizon $T$ as described in the introduction. The learner's objective is to maximize the cumulative reward over time horizon $T$ subject to $K$ anytime cumulative constraints as defined in \eqref{obj-intro}-\eqref{eq:cons-intro}.  
To address the challenges on the \emph{unknown} reward and cost in constraint, as well as the anytime cumulative constraints, we develop a pessimistic-optimistic algorithm based on the primal-dual approach for constrained optimization. We first give out the intuition of the algorithm and then provide the formal statement of the algorithm.

To start, we consider a baseline, deterministic problem that replaces all the random variables with their expectations. Different from the conventional setup, we introduce a ``tightness'' constant $\epsilon>0$: 
\begin{align}
    \max_{\mathbf x} & ~  \sum_{c\in\mathcal C,j\in[J]} p_c r(c,j) x_{c,j} 
    \label{obj-fluid-tightened}\\
    \text{s.t.} & ~ \sum_{j\in[J]}  x_{c,j} = 1,   ~x_{c,j} \geq 0, \forall c \in \mathcal C, \label{arrival-fluid-tightened} \\
    & ~ \sum_{c\in\mathcal C,j\in[J]}p_cw^{(k)}(c,j) x_{c,j} +\epsilon\leq 0, ~\forall k \in [K], \label{resource limit-fluid-tightened} 
\end{align} 

where $x_{c,j}$ can be viewed as the probability of taking action $j$ on context $c,$ and $p_c$ is the probability that context $c$ is selected in each round. We will discuss in further details the importance of the tightness constant $\epsilon$ in Section~\ref{sec:main proof}.   
The Lagrangian of the problem above is 
\begin{align}
    \max_{{\mathbf x}: \sum_{j}  x_{c,j} = 1, x_{c,j} \geq 0} \sum_{c,j} p_c r(c,j) x_{c,j}-\sum_{k}\lambda^{(k)}\left(\sum_{c\in\mathcal C,j\in[J]}p_cw^{(k)}(c,j) x_{c,j} +\epsilon \right),
\end{align} where $\lambda^{(k)}$ is the Lagrange multiplier associated with the $k$th constraint in \eqref{resource limit-fluid-tightened}. Fixing the values of the Lagrange multipliers, solving the optimization problem is equivalent to solving $|\mathcal C|$ separate subproblems \eqref{eq:subprob}, one for each context $c,$ because the optimization variables $\bf x$ are coupled through $j$ only:  
\begin{align}
    \max_{{\mathbf x}: \sum_{j}  x_{c,j} = 1, x_{c,j} \geq 0} p_c\left(\sum_{j}  r(c,j) x_{c,j}-\sum_{k}\lambda^{(k)}\left(\sum_{j}w^{(k)}(c,j) x_{c,j} \right)\right).\label{eq:subprob}
\end{align}
Since the problem above is a linear programming, one of the optimal solutions is $x_{c,j}=1$ for $j=j^*$ and $x_{c,j}=0$ otherwise, where 
\begin{equation}
j^*\in\arg\max_j  r(c,j) -\sum_{k}\lambda^{(k)} w^{(k)}(c,j)
\label{eq:maxvalue}
\end{equation} and a tie can be broken arbitrarily. If we call $r(c,j) -\sum_{k}\lambda^{(k)} w^{(k)}(c,j)$ the action-value of context $c,$ then the solution for fixed values of Lagrange multipliers is to choose an action with the highest action-value. We may view the action value here plays a similar role as the Q-value (also called action-value function) in Q-learning \cite{WatDay_92}. 

Now the challenges to find a solution according to \eqref{eq:maxvalue} include: (i) both $r(c,j)$ and $w^{(k)}(c,j)$ are unknown, and (ii) the optimal Lagrange multipliers $\lambda^{(k)}$ are also unknown. To overcome these challenges, we develop a pessimistic-optimistic algorithm that 
\begin{itemize}[leftmargin=*]
    \item Uses LinUCB to estimate $r(c,j)$ based on its linear structure. 
    \item Uses observed $W^{(k)}(c(t),j)$ to replace $w^{(k)}(c(t),j)$ at each round $t.$ 
    \item Uses the following function to dynamically approximate the Lagrange multipliers: 
    \begin{align*}
  Q^{(k)}(t+1) =& \left[Q^{(k)}(t) + \sum_{j\in [J]}W^{(k)}(c(t),j) X_{j}(t)  + \epsilon_t \right]^{+}, \forall k.
 \end{align*} In other words, we increase its value when the current cost exceeds the current ``budget,'' and decrease it otherwise. Therefore, $Q^{(k)}(t)$ keeps track of the cumulative constraint violation by round $t$.

 \item We further add a scaling parameter $1/V_t$ to $Q^{(k)}(t),$ i.e. $\frac{Q^{(k)}(t)}{V_t},$ to approximate $\lambda^{(k)}.$ With a carefully designed $V_t,$ we can control the tradeoff between maximizing reward and minimizing constraint violation in the policy and achieve the regret and constraint violation bounds to be presented in the main theorem. 
\end{itemize}

\noindent Next, we formally state our algorithm. This algorithm 
takes the following information as input at the beginning of each round $t$: (i) historical observations $${\cal F}_{t-1}=\{c(s), A(s), R(c(s),A(s)), W^{(k)}(c(s), A(s))\}_{s\in[t-1], k\in[K]},$$ (ii) current observations $c(s)$ and $\{W^{(k)}(c(s), j)\}_{k\in[K],j\in[J]},$ and (iii) system parameters: the feature map $\{\phi(c,j)\}_{c\in\mathcal C,j\in[J]},$ time horizon $T,$ and a pre-set constant $\delta$. In the analysis of our algorithm, we will reveal the connection of this constant $\delta$ with Slater's condition.   
The algorithm outputs the action in each round, observes reward $R(c(t), A(t))$, makes updates, and then moves to the next round $t+1$.

\vspace{0.1in}
\hrule
\vspace{0.1in}
\noindent{\bf A Pessimistic-Optimistic Algorithm}
\vspace{0.1in}
\hrule
\vspace{0.1in}

\noindent {\bf Initialization:} $Q^{(k)}(1)=0,$ $\mathcal B_1 = \{\theta | ||\theta||_{\Sigma_0} \leq \sqrt{\beta_1} \}, 
\Sigma_0 = \mathbf I~\text{and}~\sqrt{\beta_1} =  m + \sqrt{2\log T}.$

\noindent For $t=1,\cdots, T,$ 
\begin{itemize}[leftmargin=*]
\item \noindent {\bf Set:}  
$V_t=\delta K^{0.25}\sqrt{\frac{2t}{3}}$ and $\epsilon_t=K^{0.75}\sqrt{\frac{6}{t}}.$

\item {\bf LinUCB (Optimistic):} Use LinUCB to estimate $r(c(t),j)$ for all $j:$ $$\hat r(c(t),j) = \min\{1,\tilde r(c(t),j)\} ~~\text{with}~~ \tilde r(c(t),j) = \max_{\theta \in \mathcal B_t} \langle \theta, \phi(c(t),j)\rangle.$$

\item {\bf MaxValue:} Compute {\em pseudo-action-value} of context $c(t)$ for all action $j,$ and take the action $j^*$ with the highest pseudo-action-value, breaking a tie arbitrarily
\begin{align*}
j^*\in \arg\max_j \underbrace{\hat r(c(t),j) - \frac{1}{V_t}\sum_k W^{(k)}(c(t),j)Q^{(k)}(t)}_{\text{pseudo action value of $(c(t),j)$}}.
\end{align*}

\item {\bf Dual Update (Pessimistic):} Update the estimates of dual variables $Q^{(k)}(t)$ as follows:
\begin{align}
  Q^{(k)}(t+1) =& \left[Q^{(k)}(t) + \sum_jW^{(k)}(c(t),j) X_{j}(t)  + \epsilon_t\right]^{+}, \forall k.\label{eq:dual-update}
 \end{align}

\item {\bf Confidence Set Update:} Update $\Sigma_t,$ $\hat \theta_t,$ $\beta_{t+1}$ and $\mathcal B_{t+1}$ according to the received reward $R(c(t), j^*):$
\begin{align*}
\Sigma_t = \Sigma_{t-1} + \phi(c(t),j^*) \phi^{\dag}(c(t),j^*), ~~~~
\hat \theta_t = \Sigma^{-1}_t\sum_{s=1}^t \phi(c(s),A(s)) R(c(s), A(s)), \\
\sqrt{\beta_{t+1}} = m + \sqrt{2 \log T + d\log \left(\frac{d+t}{d}\right)}, ~~~ \mathcal B_{t+1} = \{\theta ~|~ ||\theta - \hat \theta_{t} ||_{\Sigma_{t}} \leq \sqrt{\beta_{t+1}} \}.
\end{align*}
\end{itemize}
\vspace{0.1in}
\hrule
\vspace{0.1in}

The complexity of our algorithm is similar to LinUCB. The additional complexity is proportional to the number of constraints (for updating $Q^{(k)}$), and it is much lower than OPLB in \cite{PacGhaBar_20}, where the construction of a safe policy set in each round is a major computational hurdle. Our algorithm is computationally efficient. Additionally, our algorithm does not estimate $w^{(k)}(c,j)$, hence, we do not need to make any specific assumption on $w^{(k)}(c,j).$ 

\section{Main Results: Regret and Constraint Violation Bounds}
\label{sec:main-results}

To understand the performance of a given policy $\pi,$ we will analyze both the regret and the constraint violation. For that, we first define the baselines and state the assumptions made for the performance analysis. Then, we present our main results on the regret bound and constraint violations -- for the latter, we present both results on expected violation and an additional high probability bound for pathwise constraint violation.

\subsection{Baselines and Assumptions}

\noindent{\bf Regret baseline:} We consider the following optimization problem:
\begin{align}
    \max_{\mathbf x} & ~  \sum_{c\in\mathcal C,j\in[J]} p_c r(c,j) x_{c,j} 
    \label{obj-fluid}\\
    \text{s.t.} & ~ \sum_{j\in[J]}  x_{c,j} = 1,   ~x_{c,j} \geq 0, \forall c \in \mathcal C, \label{arrival-fluid} \\
    & ~ \sum_{c\in\mathcal C,j\in[J]}p_cw^{(k)}(c,j) x_{c,j} \leq 0, ~\forall k \in [K]. \label{resource limit-fluid} 
\end{align}

\noindent{\bf Constraint violation baseline:} We choose zero (no violation) as our baseline. 

It worth noting the baseline we use in the regret analysis is derived from relaxed cumulative constraints instead of anytime cumulative constraints in the original problem. Since the cumulative constraint is the least restrictive constraint, a learner obtains the highest cumulative rewards in such a setting. In other words, our regret analysis is with respect to the best (the most relaxed) baseline. 

We make the following assumptions for all the results present in this paper. 
\begin{assumption}
The context $c(t)$ are i.i.d. across rounds. The mean reward $r(c, j)=\langle \theta_*, \phi(c,j) \rangle\in [0, 1]$ with $||\phi(c,j)||\leq 1, ||\theta_*|| \leq m$ for any $c \in \mathcal C,\  j \in [J].$  $\eta(t)$ is zero-mean $1$-subGaussian conditioned on $\{\mathcal F_{t-1}, A(t)\}$. \label{assumption:reward}
\end{assumption}

\begin{assumption}
The costs in the constraints 
satisfy $|W^{(k)}(c,j)|\leq 1.$   
Furthermore, we assume $\{W^{(k)}(c, j)\}_{t=1}^{T}$ are i.i.d. samples for given $c$ and $j$. \label{assumption:constraints}
\end{assumption}

\begin{assumption}[Slater's condition]
There exists $\delta > 0$ such that there exists a feasible solution $\bf x$ to optimization problem \eqref{obj-fluid}-\eqref{resource limit-fluid} that guarantees $\sum_{c\in\mathcal C,j\in[J]} p_cw^{(k)}(c,j) x_{c,j} \leq -\delta, \forall k \in [K].$ We assume $\delta\leq 1$ because if the condition holds for $\delta>1,$ it also holds for $\delta=1$.  \label{assumption:slater}
\end{assumption} 
We call $\delta$ Slater's constant because it comes from Slater's condition in optimization -- this is the constant used as an input our algorithm. This constant plays a similar role as the cost of a safe action in \cite{AmaAliThr_19,PacGhaBar_20}. In fact, a safe action guarantees the existence of a Slater's constant, and we can estimate the constant by running the safe action for a period of time. However, the existence of a Slater's constant does not require the existence of a safe action. It is also a more relaxed quantity than the safety gap in \cite{AmaAliThr_19}, which is defined under the optimal policy. Slater's constant can be from a feasible solution that is not necessarily optimal.

The next lemma shows that the optimal value of \eqref{obj-fluid}-\eqref{resource limit-fluid} is an upper bound on that of \eqref{obj-intro}-\eqref{eq:cons-intro}. The proof of this lemma can be found in Appendix \ref{app: baseline}. 
\begin{lemma} \label{lemma: fluid upper bound}
Assume $\{c(t)\}$ are i.i.d. across rounds, and $\{R(c, j)\}$ and $\{W^{(k)}(c, j)\}$ are i.i.d. samples given action $j$ and context $c.$ Let $\pi^*$ be the optimal policy to problem \eqref{obj-intro}-\eqref{eq:cons-intro} and $\mathbf x^*$ be the solution to \eqref{obj-fluid}-\eqref{resource limit-fluid} with entries $\{x_{c,j}^*\}_{c\in\mathcal C, j \in [J]}$. We have  
$$\mathbb E\left[\sum_{t=1}^{T}\sum_{j\in[J]} R(c(t), j)X^{\pi^*}_{j} (t)\right] \leq T\sum_{c\in\mathcal C,j\in[J]} p_c r(c,j) x^*_{c,j}.$$
\end{lemma}
\noindent The baseline problem \eqref{obj-fluid}-\eqref{resource limit-fluid}  is the same as the one presented in Section~\ref{sec:alg} except that the tightness constant $\epsilon = 0$ here. Any feasible solution for the tightened problem in Section~\ref{sec:alg} is a feasible solution to the baseline problem. Under Slater's condition, when $\epsilon<\delta,$ the tightened problem also has feasible solutions.

\subsection{Regret and Constraint Violation Bounds}
Given the baselines above, we now define regret and constraint violation.

\noindent{\bf Regret:} Given policy $\pi,$ we define the (pseudo)-regret of  the policy to be 
\begin{align}
\mathcal R(\tau) =& \tau \sum_{c\in\mathcal C,j\in[J]} p_c r(c,j) x^*_{c,j}  - \mathbb E\left[\sum_{t=1}^{\tau} \sum_{j\in [J]}R(c(t), j) X^{\pi}_{j}(t)\right]. \label{def:regret}
\end{align}

\noindent{\bf Constraint violation:} The constraint violation in round $\tau$ is defined to be
\begin{align}
    \mathcal V(\tau) = \sum_{k\in[K]} \left(  \mathbb E\left[\sum_{t=1}^{\tau}\sum_{j\in[J]}W^{(k)}(c(t),j)X^{\pi}_{j}(t)\right] \right)^+. \label{def:cv}
\end{align}
Note that the operator $(\cdot)^+=\max(\cdot, 0)$ is imposed so that different types of constraint violations will not be canceled out. 

\begin{theorem}
Under Assumptions \ref{assumption:reward}-\ref{assumption:slater}, the pessimistic-optimistic algorithm  
achieves the following regret and constraint violations bounds for any $\tau \in [T]:$ 
\begin{align*}
\mathcal R(\tau) \leq& \frac{60K^{3}}{\delta^3}+ \frac{4\sqrt{6}K^{0.75}\sqrt{\tau}}{\delta} + 2 +\sqrt{8d\tau \beta_{\tau}(T^{-1}) \log \left(\frac{ d + \tau}{ d}\right)},\\
\mathcal V(\tau) \leq&  K^{1.5}\left(\frac{48K^2}{\delta}\log\left(\frac{16}{\delta}\right) + \frac{24K^{1.5}}{\delta^2}+\frac{30K^{1.5}}{\delta}+ 8K - \sqrt{\tau} \right)^{+}. 
\end{align*}
where $\sqrt{\beta_{\tau}(T^{-1})} = m+\sqrt{2\log T + d \log \left(1+\tau/d\right)}.$
\label{thm:formal-UCB}
\end{theorem}

We make a few important observations from our theoretical results. First, for the \textit{reward regret}, we observe
$${\cal R}(\tau)=\tilde{\cal O}\left(\left(\frac{K^{0.75}}{\delta}+d\right)\sqrt{\tau}\right).$$
So the regret is independent of the number of contexts, action space $[J]$ and the dimension of cost functions $W^{(k)}(\cdot,\cdot).$  It grows sub-linearly in $\tau$ and the number of constraints $K,$ and linearly in the dimension of reward feature $d$ and the inverse of Slater's constant $\delta.$ 

Second, for the \textit{constraint violation}, we observe 
\begin{equation}
{\cal V}(\tau)=\begin{cases}
\tilde{\cal O}\left( \frac{K^{3.5}}{\delta} + \frac{K^{3}}{\delta^2} \right)&\tau\leq \tau'=\tilde{\cal O}\left(\frac{K^4}{\delta^2}+\frac{K^3}{\delta^4}\right)\\
0&\hbox{otherwise}
\end{cases}.
\end{equation}
That is, the constraint violation is {\em independent} of horizon $T$ and becomes {\em zero} when $\tau > \tau'.$ The constraint violation, however, has a strong dependence on $K$ and $\delta$ when $\tau\leq \tau'.$ This is not surprising because $K$ defines the number of constraints, and $\delta$ represents the tightness of the constraints (or size of the feasible set).

\noindent\textbf{Dependence on Slater's constant. } Both the regret and constraint violation increase in $\delta$. To see the intuition, note that $\delta$ determines the size of the feasible set for the optimization problem. A larger $\delta$ implies a larger feasible set, so it is easier to find a feasible solution, vice versa. Therefore, both regret and constraint violation increase as $\delta$ decreases because the problem becomes harder and requires more accurate learning. 

\noindent\textbf{Sharpness of the bound.}
In terms of horizon $T,$ the bounds in Theorem \ref{thm:formal-UCB} are sharp because the regret bound $\mathcal R(T)$  matches the instance-independent regret $\Omega(\sqrt{T})$ in multi-armed bandit problems without constraints \cite{AueBiaFre_95, GerLat_16} up to logarithmic factors. Furthermore, zero constraint violation is the best possible. Therefore, the bounds are sharp up to logarithmic factors in terms of horizon $T$. It is not clear whether these bounds are sharp in terms of $K$, $d$, and $\delta$, which are interesting open questions.

\subsection{A High Probability Bound on Constraint Violation}
The constraint \eqref{eq:cons-intro} defined in the original problem and the constraint violation measure defined in \eqref{def:cv} are both in terms of expectation. An interesting, related question is what the probability is for a {\em sample-path version} of the constraints to be satisfied. It turns out that our algorithm provides a high probability guarantee on that as well. The proof can be found in Appendix \ref{app: tail prob}. 
\begin{corollary}\label{cor: tail prob}
The pessimistic-optimistic algorithm guarantees that for any $\tau \geq \frac{\kappa K^5}{\delta^2}\left(\log\left(\frac{K}{\delta}\right)\right)^2$, where $\kappa$ is a positive constant independent of $\tau,$ $K,$ $\delta$ and $d$,  
$$\mathbb P\left(\sum_{t=1}^{\tau}\sum_{j\in[J]}W^{(k)}(c(t),j)X_{j}(t)>0\right)=O\left(e^{-\frac{\delta\sqrt{\tau}}{50K^{2}}}\right).$$   
\end{corollary}

\section{ Proof of Theorem \ref{thm:formal-UCB}}\label{sec:main proof}

We first explain the intuition behind the main result. Recall that the algorithm selects action $j^*$ such that 
\begin{align*}
j^* \in \arg\max_j \left(\hat r(c(\tau),j) - \frac{1}{V_\tau}\sum_k W^{(k)}(c(\tau),j)Q^{(k)}(\tau)\right),
\end{align*} and $V_\tau=O(\sqrt{\tau}).$ Therefore, when $Q^{(k)}(\tau)=o(\sqrt{\tau}),$ the reward term dominates the cost term, and our algorithm uses LinUCB to maximize the reward.  When $Q^{(k)}(\tau)=\omega(\sqrt{\tau}),$  the cost term dominates the reward term and our algorithm focuses on reducing $Q^{(k)}$. Slater's condition implies that there exists a policy that can reduce $Q^{(k)}$ by a constant (related to $\delta$) in each round. Therefore, the algorithm takes $\tilde{\cal O}(\sqrt{\tau})$ rounds to reduce $Q^{(k)}$ to $o(\sqrt{\tau}),$ which may add $\tilde{\cal O}(\sqrt{\tau})$ to the regret during this period. The argument above also implies that $Q^{(k)}(\tau)=O(\sqrt{\tau}).$ Then, because 
\begin{align}
\mathbb E\left[\sum_{t=1}^\tau \sum_{j\in [J]}W^{(k)}(c(t),j) X_{j}(t)\right]\leq \mathbb E\left[Q^{(k)}(\tau+1)\right]-\sum_{t=1}^\tau \epsilon_t.\nonumber
\end{align}
we can further bound the constraint violation at time $\tau$ to be a constant or even zero via the bound on $\mathbb E[Q^{(k)}(\tau+1)]$ and a proper choice of $\epsilon_t.$ 

\subsection{Regret Bound}\label{sec: regret}
Now consider the regret defined in \eqref{def:regret} and define $\mathbf x^{\epsilon_t,*}$ to be the optimal solution to the tightened problem  \eqref{obj-fluid-tightened}-\eqref{resource limit-fluid-tightened} with  $\epsilon=\epsilon_t.$ We obtain the following decomposition by adding and subtracting corresponding terms: 
\begin{align} 
{\cal R}(\tau)=&\tau \sum_{c,j}p_cr(c,j)x_{c,j}^* - \mathbb E\left[\sum_{t=1}^{\tau} \sum_{j}R(c(t), j)X_{j}(t)\right] \nonumber\\
\stackrel{(a)}{=}&\tau \sum_{c,j}p_cr(c,j)x_{c,j}^* - \mathbb E\left[\sum_{t=1}^{\tau} \sum_{j}r(c(t), j)X_{j}(t)\right] \nonumber\\
=&   
\underbrace{\sum_{t=1}^\tau \sum_{c,j}p_cr(c,j)\left(x_{c,j}^*-x_{c,j}^{\epsilon_t,*}\right)}_{\text{$\epsilon_t$-tight}} 
+ \underbrace{\mathbb E\left[\sum_{t=1}^{\tau} \sum_{j} \left(r(c(t),j)-\hat r(c(t),j)\right)x_{c(t),j}^{\epsilon_t,*}\right]}_{\text{reward mismatch}}
\nonumber\\
&+\underbrace{\mathbb E\left[ \sum_{t=1}^{\tau} \sum_{j}\left(\hat r(c(t),j)x_{c(t),j}^{\epsilon_t,*} - \hat r(c(t),j)X_{j}(t)\right)\right] - \sum_{t=1}^{\tau} \frac{K(1+\epsilon_t^2)}{V_t}}_{\text{Lyapunov drift}} \label{eq:driftterm}\\
&+\underbrace{\sum_{t=1}^{\tau} \frac{K(1+\epsilon_t^2)}{V_t}}_{\text{accumulated tightness}}+ \underbrace{ \mathbb E\left[\sum_{t=1}^{\tau}\sum_{j} \left(\hat r(c(t),j)-r(c(t),j)\right)X_{j}(t)\right]}_{\text{reward mismatch}},  \label{intuition}
\end{align} where $(a)$ holds because the random reward is revealed after action $A(t)$ is taken so the noise is independent of the action. 

We next present a sequence of lemmas that bounds the terms above. The proofs of these lemmas are presented in Appendices \ref{app: drif in regret}, \ref{app: epsilon gap}, and \ref{app: linucb-r}. The key to the proof, particularly the proof of Lemma~\ref{lemma: drift in regret}, relies on the Lyapunov-drift analysis; a comprehensive introduction of the method can be found in \cite{Nee_10,SriYin_14}. 

\begin{lemma}\label{lemma: drift in regret}
Under the Pessimistic-Optimistic Algorithm, we have 
\begin{align}
\mathbb E\left[\sum_{t=1}^{\tau}\sum_{j\in [J]} \hat r(c(t),j)  \left(x^{\epsilon_t}_{c(t),j}-X_{j}(t)\right) \Big| \mathbf H(t)=\mathbf h\right] - \sum_{t=1}^{\tau}\frac{K(1+\epsilon_t^2)}{V_t} \leq \sum_{t=1}^{\tau}\frac{ Kt(\epsilon_t+\epsilon_t^2) \mathbb I(\epsilon_t > \delta)}{V_t}. \nonumber
\end{align}  
\end{lemma}

\begin{lemma}\label{lemma: elsilon gap}
Under Assumptions \ref{assumption:reward}-\ref{assumption:slater}, we can bound the difference between the baseline optimization problem and its tightened version: 
$$\sum_{t=1}^\tau \sum_{c,j}p_cr(c,j)\left(x_{c,j}^*-x_{c,j}^{\epsilon_t,*}\right) \leq \sum_{t=1}^{\tau}\frac{\epsilon_t}{\delta}.$$
\end{lemma}
\begin{lemma}\label{lemma: bandits}
Under the Pessimistic-Optimistic Algorithm, LinUCB guarantees that 
\begin{align*}
\mathbb E\left[\sum_{t=1}^{\tau}\sum_{j}(\hat r(c(t),j)-r(c(t),j))X_{j}(t)\right] \leq&  1 + \sqrt{8d\tau \beta_{\tau}(T^{-1}) \log \left(\frac{ d + \tau}{d}\right)},\\
\mathbb E\left[ \sum_{t=1}^{\tau} \sum_{j}(r(c(t),j) - \hat r(c(t),j))x^{\epsilon_t,*}_{c(t),j}\right] \leq& 1.
\end{align*}
\end{lemma}

Based on Lemmas \ref{lemma: drift in regret}, \ref{lemma: elsilon gap}, and \ref{lemma: bandits}, we conclude that
\begin{align*}
{\cal R}(\tau)\leq &  \sum_{t=1}^{\tau} \frac{Kt(\epsilon_t+\epsilon_t^2)\mathbb I(\epsilon_t > \delta)}{V_t} + \sum_{t=1}^{\tau}\frac{\epsilon_t}{\delta} + \sum_{t=1}^{\tau}\frac{K(1+\epsilon_t^2)}{V_t} 
+ 2 + \sqrt{8d\tau \beta_{\tau}(T^{-1}) \log \left(\frac{ d + \tau}{ d}\right)}.
\end{align*}
By choosing $\epsilon_t=K^{0.75}\sqrt{\frac{6}{t}}$ and $V_t=\delta K^{0.25}\sqrt{\frac{2t}{3}},$ we have
$$\sum_{t=1}^\tau \epsilon_t \leq 2K^{0.75}\sqrt{6\tau},~ \sum_{t=1}^\tau1/V_t\leq \frac{\sqrt{6\tau}}{K^{0.25}},~\text{and}~\sum_{t=1}^{\tau} \frac{Kt(\epsilon_t+\epsilon_t^2) \mathbb I(\epsilon_t > \delta)}{V_t}\leq \frac{60K^{3}}{\delta^3},$$ which yields the regret bound
\begin{align*}
\mathcal R(\tau) \leq& \frac{60K^{3}}{\delta^3}+ \frac{4\sqrt{6}K^{0.75}\sqrt{\tau}}{\delta} + 2 +\sqrt{8d\tau \beta_{\tau}(T^{-1}) \log \left(\frac{ d + \tau}{ d}\right)}.
\end{align*}

\subsection{Constraint Violation Bound}\label{sec: cv}
According to the dynamic defined in \eqref{eq:dual-update}, we have
\begin{align*}
Q^{(k)}(\tau+1) \geq \sum_{t=1}^\tau\sum_jW^{(k)}(c(t),j) X_{j}(t) + \sum_{t=1}^\tau  \epsilon_t,
\end{align*} where we used the fact $Q^{(k)}(0)=0.$ This it implies the constraint violation can be bounded as follows: 
\begin{align}
\mathcal V(\tau) \leq& \sum_{k}\left(\mathbb E[Q^{(k)}(\tau+1)] - \sum_{t=1}^{\tau} \epsilon_t\right)^{+}. \label{eq:cv-q} 
\end{align} 

Next, we introduce a lemma on the upper bound of $Q^{k}(\tau).$ Define $\tau'$ the first time such that $\epsilon_{\tau'} \leq \delta/2,$ that is, $\epsilon_\tau > \delta/2, \forall 1\leq \tau < \tau'.$ Note that $Q^{(k)}(\tau') \leq \sum_{t=1}^{\tau'}(1+\epsilon_t)$ because $Q^{(k)}$ can increase by at most $(1+\epsilon_t)$ during each round. 
\begin{lemma}\label{lemma: Q bound} 
For any time $\tau\in [T]$ such that $\tau \geq \tau',$ i.e., $\epsilon_\tau \leq \delta/2,$ we have
\begin{align*}
\mathbb E\left[Q^{(k)}(\tau)\right]\leq \sqrt{K}\left(\frac{48K^2}{\delta}\log\left(\frac{16K}{\delta}\right)+2K+\frac{4(V_\tau + K(1+\epsilon_\tau^2))}{\delta} + \tau' +  \sqrt{K}\sum_{t=1}^{\tau'}\epsilon_t\right).
\end{align*}
\end{lemma}
Based on our choices of $\epsilon_t=K^{0.75}\sqrt{\frac{6}{t}}$ and $V_t=\delta K^{0.25}\sqrt{\frac{2t}{3}},$ we obtain 
$$\tau' = \frac{24K^{1.5}}{\delta^2} ~\text{and}~\sum_{t=1}^\tau \frac{\epsilon_t}{\sqrt{K}} - \frac{4V_{\tau+1}}{\delta}\geq \sqrt{\tau} - 6K.$$ 
Note $\mathcal V(\tau) \leq K\tau'$ for any $\tau \leq \tau'.$ Combine Lemma \ref{lemma: Q bound} into \eqref{eq:cv-q}, we conclude
\begin{align*}
\mathcal V(\tau) 
\leq& K^{1.5}\left(\frac{48K^2}{\delta}\log\left(\frac{16}{\delta}\right) + \frac{24K^{1.5}}{\delta^2}+\frac{30K^{1.5}}{\delta}+ 8K - \sqrt{\tau} \right)^{+}. 
\end{align*} 

\section{Numerical Evaluations}
In this section, we present numerical evaluations of the proposed algorithm, including 1) the constrained multi-armed bandit (MAB) example studied in \cite{PacGhaBar_20}; 2) a constrained linear bandit example based on a healthcare dataset on inpatient flow routing. 

\noindent{\bf The Constrained MAB Example in \cite{PacGhaBar_20}:} As acknowledged in \cite{PacGhaBar_20}, the proposed algorithm OPA in \cite{PacGhaBar_20} suffers from high computational complexity for linear bandits so they evaluated the performance of their algorithm with classical multi-armed bandits (MAB), for which a computationally efficient algorithm, called OPA, is proposed (Lemma 5 in \cite{PacGhaBar_20}). Therefore, we compared our algorithm with OPA by considering the MAB example in \cite{PacGhaBar_20} with $K = 4$-arms where the reward and cost distributions are Bernoulli with means $\bar r = (0.1, 0.2, 0.4, 0.7)$ and $\bar c = (0, 0.4, 0.5, 0.2)$ and the total cost in each round should not exceed $0.5.$ In particular, we set parameters of our algorithms with $V_t=\sqrt{t}, \epsilon_t = 6.0/\sqrt{t}$ and UCB bonus terms $\sqrt{\log t/N(k)}$ for arm-$k$ at time $t.$   
The results are presented in Figure \ref{fig:com with OPA}, where we can see that our algorithm has significant lower regret than that under OPA while the cost constraints are satisfied under both algorithms. 
\begin{figure}[H]
\centering
\begin{subfigure}{0.32\textwidth}
    \includegraphics[width=\textwidth]{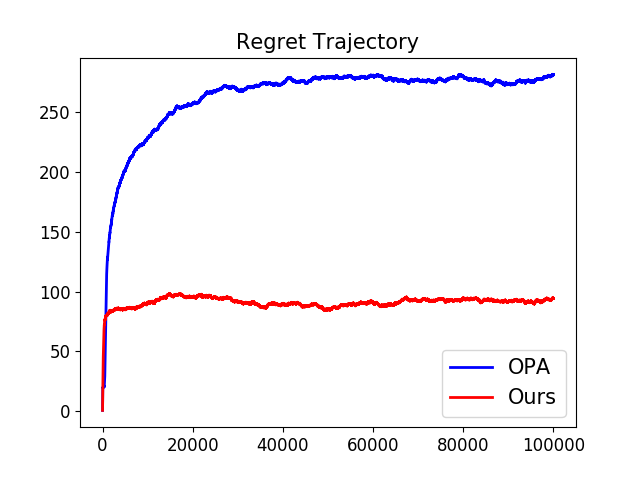}
    \caption{Regret}
    \label{fig:first}
\end{subfigure}
\hfill
\begin{subfigure}{0.32\textwidth}
    \includegraphics[width=\textwidth]{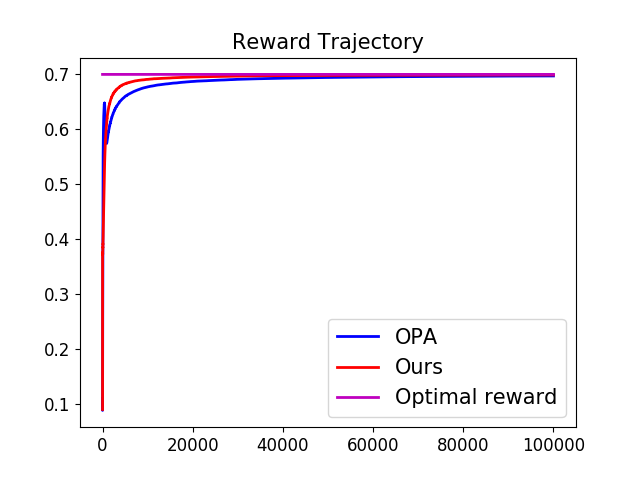}
    \caption{Reward}
    \label{fig:second}
\end{subfigure}
\hfill
\begin{subfigure}{0.32\textwidth}
    \includegraphics[width=\textwidth]{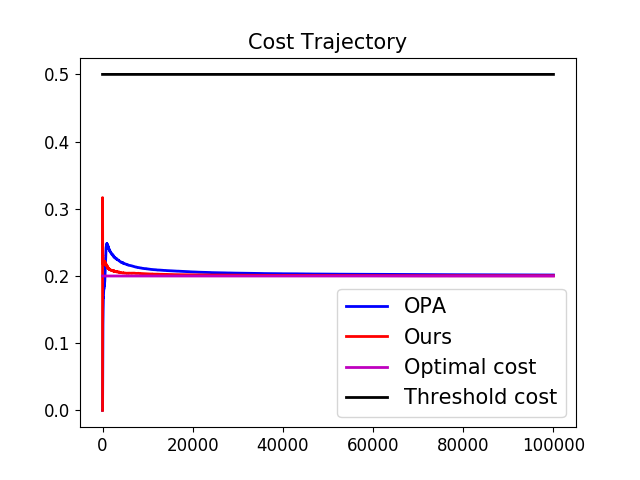}
    \caption{Cost}
    \label{fig:third}
\end{subfigure}
\caption{Our Algorithm versus OPA in \cite{PacGhaBar_20}.}
\label{fig:com with OPA}
\end{figure}

\noindent{\bf Constrained Linear Bandits for Inpatient Flow Routing:} We also evaluated our algorithm  for inpatient flow routing on a real-world dataset, where incoming patients have different features (context) such as age, gender, medical history, etc, and incur different amounts of ``rewards'' when being assigned to different wards (actions). The rewards are different because the levels of care provided by different wards match with the patients' needs differently. In this evaluation, we measured the reward via the avoided 30-day readmission penalty, i.e., a reward is collected if the patient is not readmitted to the hospital within 30 days since being discharged. Calibrating from the data we have, 
we considered three types of constraints: capacity, fairness, and resource. 
After normalizing, the capacity constraint for each ward is $[0.2, 0.2, 0.175, 0.175, 0.175, 0.175],$ the fairness requirement is $[0.175,0.175,0.15,0.15,0.125,0.125],$ and the nursing resource constraint is $[0.1875,0.1875,0.1875,0.1875,0.1875,0.1875]$, where each patient consumes ``one" unit of resources after being assigned to a ward. We note that these constraints are strict since the hospital capacity is highly constrained. 
In the experiment, we set scaling parameter $V_t = 4\sqrt{t}, \epsilon_t = 1/\sqrt{t}$ in our algorithm, and we ran various learning horizons $T = [50^2, 100^2, 150^2, 200^2, 250^2].$ 
The regrets and constraint violations at the end of the horizon are plotted in Figure \ref{fig:healthcare-regret} and \ref{fig:healthcare-violation}, which show that our algorithm achieves a low regret and zero violation. 
To further evaluate the anytime constraint violation, we plotted a representative trajectory with $T=10,000$ in Figure \ref{fig:healthcare-trace} to see how the violations evolve and if zero constraint violation can be achieved after a constant number of steps. The results show that the constraint violations decrease to zero quickly and $\tau'$ for capacity, fairness and resource constraints are $(32,50,54),$ respectively, which confirms our theoretical results on zero constraint violation. 
\begin{figure}[H]
\centering
\begin{subfigure}{0.32\textwidth}
    \includegraphics[width=\textwidth]{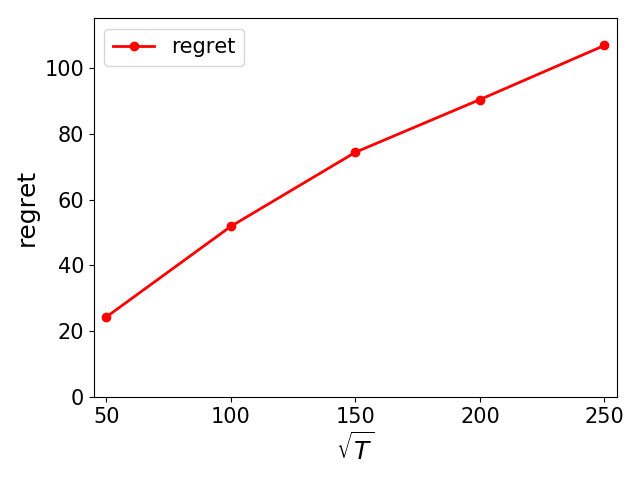}
    \caption{Regret}
    \label{fig:healthcare-regret}
\end{subfigure}
\hfill
\begin{subfigure}{0.32\textwidth}
    \includegraphics[width=\textwidth]{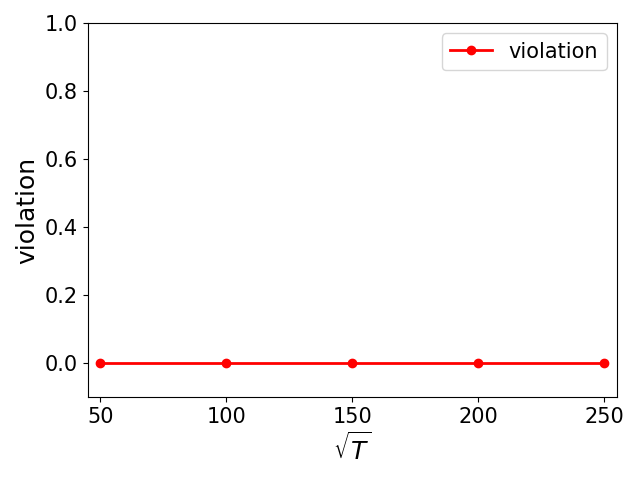}
    \caption{Violation}
    \label{fig:healthcare-violation}
\end{subfigure}
\hfill
\begin{subfigure}{0.32\textwidth}
    \includegraphics[width=\textwidth]{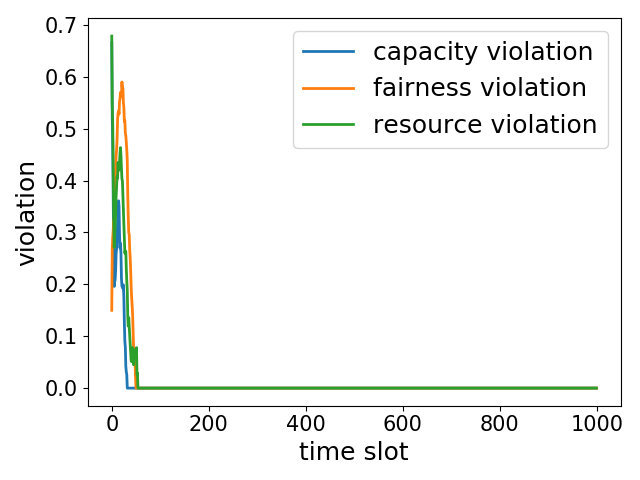}
    \caption{Trajectory of Violation}
    \label{fig:healthcare-trace}
\end{subfigure}
\caption{Regret and Violation in Inpatient  Flow  Routing.}
\label{fig:healthcare}
\end{figure}

\section{Conclusions and Extensions}
\label{sec:conclusion}
In this paper, we study stochastic linear bandits with general anytime cumulative constraints. We develop a pessimistic-optimistic algorithm that is computationally efficient and has strong guarantees on both regret and  constraint violations. We conclude this paper by mentioning an extension on the case where the cost signals $W^{(k)}(c(t),A(t))$ are revealed after action $A(t)$ is taken. However, we assume the costs can be linearly parameterized as in \cite{PacGhaBar_20}. 

\noindent{\bf Linear Cost Functions:} In this case, we assume that the costs $,W^{(k)}(c(t),A(t)),$ are not available before the action is taken but the costs are linear as in \cite{PacGhaBar_20}.  The learner observes cost \emph{after} taking action $A(t),$  $$W(c(t), A(t))= \langle \mu_{*}, \psi(c, j) \rangle + \xi(t),$$ where $\psi(c,j)\in \mathbb R^d$ is a $d$-dimensional feature vector for (context, action) pair $(c,j),$ $\mu_{*} \in \mathbb R^d$ is an unknown underlying vector to be learned, and $\xi(t)$ is a zero-mean random variable. In this case, we also obtain an estimate $\check W(c(t), j)$ of $W(c(t), j)$ with LinUCB and replace $W(c(t), j)$ with $\check W(c(t), j)$ in the steps of MaxValue and Dual Update in the Pessimisitic-Optimistic Algorithm. This variation of Pessimisitic-Optimistic Algorithm has a similar computational complexity as our main algorithm, and it can provide similar regret and constraint violation guarantees:

\begin{theorem}[Informal]
With linear costs as in \cite{PacGhaBar_20}, a variation of our algorithm
achieves ${\cal R}(\tau)=\tilde{\cal O}\left(\frac{d}{\delta}\sqrt{\tau}+\frac{d^4}{\delta^3}\right)$ for any $\tau\in [T]$ and ${\cal V}(\tau)=0$ for $\tau\geq \tau''=O\left(\frac{d^2}{\delta^4}\log^2 T\right).$\label{thm:linearcost}
\end{theorem}
The formal statement and the detailed analysis are in Appendix \ref{app:linearcost}.

\newpage
\bibliographystyle{plain}
\bibliography{inlab-refs}

\newpage
\appendix
\allowdisplaybreaks[0]

\section{Summary of Notation}\label{app:notation}
We summarize  notations used in the paper in Table \ref{tab}. 
\begin{table}[h]
\centering
\begin{tabular}{|c|c|}
\hline
Notation                       & Definition                                                            \\ \hline
$T$                            &  Horizon                                                \\ \hline
$K$                            & Total number of constraints                                                 \\ \hline
$\mathcal C$         & The context set                                                            \\ \hline
$c(t)$         & The context in round $t$                                                           \\ \hline
$[J]$         & The action set                                                            \\ \hline
$A(t)$         & The action taken in round $t$                                                            \\ \hline
$d$                            & The dimension of the feature space                           \\ \hline
$R(c, j)$                      & The reward received with context $c$ and action $j$                      \\ \hline
$r(c, j)$                      & $\mathbb E[R(c,j)]$                      \\ \hline
$W^{(k)}(c, j)$                & The type-$k$ cost associated with  context $c$ and action $j$   \\ \hline
$w^{(k)}(c, j)$                      & $\mathbb E[W^{(k)}(c,j)]$                      \\ \hline
$w^{(k)}(c, j)$                      & $\mathbb E[W^{(k)}(c,j)]$                      \\ \hline
$\delta$                & Slater's constant \\ \hline
$V_t$                          & A time varying weight                                \\ \hline
$Q^{(k)}(t)$                   & The dual estimation of the $k$th Lagrange multiplier in round $t$                \\ \hline
$\epsilon_t$                   & The time varying tightness parameter                    \\ \hline
$\theta_*$ and $\hat \theta_t$ & The true reward parameters and its estimation in round $t$ with $||\theta_*||\leq m$           \\ \hline
$\mathcal B_t$   & The confidence set of $\theta_*$ in round $t$         \\ \hline
$\beta_t(p)$   & The radius of confidence set in round $t$ with $\sqrt{\beta_t(p)} = m + \sqrt{2\log\frac{1}{p}+d\log(1+t/d)}$       \\ \hline
$\Sigma_t$ & The covariance  matrix in round $t$         \\ \hline
\end{tabular}
\caption{Notation Table}
\label{tab}
\end{table}

\section{Proof of Lemma \ref{lemma: fluid upper bound}}\label{app: baseline}
Recall $\{c(t)\}_{t=1}^{T}$ is i.i.d. across rounds and  $R(c, j)$ ($W^{(k)}(c, j)$) are i.i.d. samples when action $j$ is taken on context $c.$ Let $\hat \pi^*$ be the optimal policy to the following problem: 
\begin{align} 
    \max_{\pi } & ~  \mathbb E\left[\sum_{t=1}^{T}\sum_{j=1}^{J} R(c(t), j)X^\pi_{j} (t)\right] \label{obj-T-relax}
    \\
    \text{s.t.:}
    & ~ \mathbb E\left[\sum_{t=1}^{T} \sum_{j=1}^JW^{(k)}(c(t), j) X^\pi_{j}(t)\right] \leq 0,  ~ \forall  k \in [K].
    \label{resource limit-T-relax} 
\end{align}
This is a relaxed version of \eqref{obj-intro}-\eqref{eq:cons-intro} because we only impose the constraints at the end of horizon $T$ instead of in each round in \eqref{obj-intro}-\eqref{eq:cons-intro}.

We have 
\begin{align}
\mathbb E\left[\sum_{t=1}^{T}\sum_{j\in[J]} R(c(t), j)X^{\hat \pi^*}_{j} (t)\right] =& \sum_{t=1}^{T}\mathbb E\left[\mathbb E\left[\sum_{j\in[J]} R(c(t), j)X^{\hat \pi^*}_{j} (t) | \mathcal F_{t-1}\right] \right] \nonumber\\
=& \sum_{t=1}^{T}\mathbb E\left[\sum_{c\in \mathcal C, j\in[J]}p_c r(c,j) \mathbb E\left[\left.X^{\hat \pi^*}_{j} (t) \right| c(t)=c, \mathcal  F_{t-1}\right] \right] \nonumber\\
=& \mathbb E\left[\sum_{c\in \mathcal C, j\in [J]}p_c r(c, j)\sum_{t=1}^{T}\mathbb E\left[X^{\hat \pi^*}_{j} (t)|c(t)=c, \mathcal F_{t-1} \right]\right] \nonumber\\
=& \sum_{c\in \mathcal C, j\in [J]}p_c r(c, j)\sum_{t=1}^{T}\mathbb E\left[\left.X^{\hat \pi^*}_{j} (t)\right|c(t)=c\right] \nonumber
\end{align}
Similarly, we have
\begin{align}
\mathbb E\left[\sum_{t=1}^{T} \sum_{j\in[J]}W^{(k)}(c(t), j) X^{\hat \pi^{*}}_{j}(t)\right]
=&\sum_{c\in \mathcal C, j\in [J]} p_c w^{(k)}(c, j)  \sum_{t=1}^{T}\mathbb E\left[\left.X^{\hat \pi^*}_{j}(t)\right|c(t)=c\right]. \nonumber
\end{align}

Define $$x^{\hat \pi^*}_{c,j}=\frac{1}{T} \sum_{t=1}^{T}\mathbb E\left[\left.X^{\hat \pi^*}_{j} (t)\right|c(t)=c\right].$$ $x^{\hat \pi^*}_{c,j}$ is a feasible solution to \eqref{obj-fluid}-\eqref{resource limit-fluid} because $\hat{\pi}^*$ is a feasible solution to \eqref{obj-T-relax}-\eqref{resource limit-T-relax}. Therefore, we have 
\begin{align*}
\mathbb E\left[\sum_{t=1}^{T}\sum_{j=[J]} R(c(t), j)X^{\pi^*}_{j} (t)\right] \leq &\mathbb E\left[\sum_{t=1}^{T}\sum_{j\in[J]} R(c(t), j)X^{\hat \pi^*}_{j} (t)\right]  \\
= & T\sum_{c\in\mathcal C,j\in[J]} p_c r(c,j) x^{\hat \pi^*}_{c,j}\\
\leq & T\sum_{c\in\mathcal C,j\in[J]} p_c r(c,j) x^*_{c,j},
\end{align*} where the first inequality holds because \eqref{obj-T-relax}-\eqref{resource limit-T-relax} is a relaxed version of \eqref{obj-intro}-\eqref{eq:cons-intro} and the last inequality holds because ${\bf x}^{\hat \pi^*}$ is a feasible solution while ${\bf x}^{*}$ is the optimal solution. 

\section{Proof of Lemma \ref{lemma: drift in regret}}\label{app: drif in regret}
Define a Lyapunov function to be $$L(t) = \frac{1}{2}\sum_k \left(Q^{(k)}(t)\right)^2.$$ 
We first obtain the Lyapunov drift 
\begin{align*}
&L(t+1)-L(t) \\
\leq&  \sum_{k}Q^{k}(t)\left(\sum_jW^{(k)}(c(t),j) X_{j}(t)+\epsilon_t\right) + \frac{\sum_{k} \left(\sum_{j}W^{(k)}(c(t),j) X_{j}(t) +\epsilon_t\right)^2}{2}\\
\leq& \sum_{k}Q^{k}(t)\left(\sum_jW^{(k)}(c(t),j) X_{j}(t)+\epsilon_t\right) -  V_t \sum_{j} \hat r(c(t),j) X_{j}(t) \\
&+ V_t \sum_{j} \hat r(c(t),j) X_{j}(t) + K(1+\epsilon_t^2)
\end{align*}
where the first inequality holds because $$ Q^{(k)}(t+1) = \left(Q^{(k)}(t) + \sum_jW^{(k)}(c(t),j)X_{j}(t) + \epsilon_t \right)^{+}$$ and the second inequality holds under Assumption \ref{assumption:constraints}. 

We next obtain the expected Lyapunov drift conditioned on the current state $\mathbf H(t)=[\mathbf Q(t),  \hat{\mathbf r}(t)]=\mathbf h=[\mathbf Q, \hat{\mathbf r}],$ where $\{\bf Q\}(t)$ are the values of the dual estimates in round $t$ and $\hat{\mathbf r}(t)$ are the reward estimates in round $t,$ 
\begin{align}
& \mathbb E[L(t+1)-L(t)|\mathbf H(t)=\mathbf h] \nonumber\\
\leq& \mathbb E\left[\sum_{k}Q^{(k)}(t) \left(\sum_jW^{(k)}(c(t),j)X_{j}(t)+\epsilon_t\right)  -  V_t \sum_{j} \hat r(c(t),j)  X_{j}(t) \Big|\mathbf H(t)=\mathbf h\right] \nonumber\\
&+ V_t \mathbb E\left[\sum_{j} \hat r(c(t),j)X_{j}(t)\Big|\mathbf H(t)=\mathbf h\right] + K(1+\epsilon_t^2).\nonumber
\end{align}  

The following lemma bounds the first term. 
\begin{lemma}
\label{lemma: drift-inequality}
Given that $\mathbf x^{\epsilon}$ is a feasible solution to \eqref{obj-fluid-tightened}-\eqref{resource limit-fluid-tightened} and $\mathbf X(t)$ is the action under the Pessimistic-Optimistic Algorithm, we have
\begin{align}
&\mathbb E\left[\sum_{k}Q^{(k)}(t) 
\left(\sum_jW^{(k)}(c(t),j) X_{j}(t) +\epsilon\right)  -  V_t \sum_{j} \hat r(c(t),j) X_{j}(t)\Big|\mathbf H(t)=\mathbf h\right] 
\nonumber\\
\leq&\mathbb E\left[\sum_{k}Q^{(k)}(t)
\left(\sum_jW^{(k)}(c(t),j) x^{\epsilon}_{c(t),j}+\epsilon\right)  -  V_t \sum_{j}\hat r(c(t),j) x^{\epsilon}_{c(t),j}\Big|\mathbf H(t)=\mathbf h\right] \nonumber
\end{align}
\end{lemma}
\begin{proof} Since $\mathbf x^{\epsilon}$ is a feasible solution to \eqref{obj-fluid-tightened}-\eqref{resource limit-fluid-tightened}, we have $\sum_{j} x^{\epsilon}_{c,j}=1, \forall c.$ Define $$\mathbb P\left(c(t)=c, \mathbf W(c(t))=W\right)=p_{c,W},$$ where $\mathbf W(c(t))$ is a $K\times J$ matrix such that $\mathbf W_{k,j}(c(t))=W^{(k)}(c(t),j),$ 
and
 $$j^*_c \in \argmax_j V_t\hat r(c,j) - \sum_k Q^{(k)} W^{(k)}(c,j).$$ 

Conditioned on $\mathbf H(t)=\mathbf h,$
we have
\begin{align}
&E\left[V_t \sum_{j}\hat r(c(t),j) x^{\epsilon}_{c(t),j}-\sum_{k}Q^{(k)}(t)
\left(\sum_jW^{(k)}(c(t),j) x^{\epsilon}_{c(t),j} +\epsilon\right)  \Big|\mathbf H(t)=\mathbf h\right]\nonumber\\
=&\sum_{c,W} p_{c,W} \sum_j\left(V_t\hat r(c,j)  -\sum_{k}Q^{(k)}W_{k,j}\right)x^{\epsilon}_{c,j} - \sum_{k} Q^{(k)} \epsilon \nonumber \\
\leq_{(a)}& \sum_{c,W} p_{c,W} \left(V_t\hat r(c,j^*_c)  -\sum_{k}Q^{(k)}W_{k,j^*_c}\right)- \sum_{k} Q^{(k)}\epsilon \nonumber \\
=_{(b)}& \mathbb E\left[ \sum_j\left(V_t \hat r(c(t),j)-\sum_{k}Q^{(k)}(t)W^{(k)}(c(t),j)\right)X_j(t) \Big| \mathbf H(t) = \mathbf h \right] - \sum_{k}Q^{(k)}\epsilon  \nonumber \\
=& \mathbb E\left[ \sum_{j} V_t \hat r(c(t),j) X_{j}(t)-\sum_{k}Q^{(k)}(t)\left(\sum_{j}W^{(k)}(c(t),j) X_{j}(t) + \epsilon\right) \Big| \mathbf H(t) = \mathbf h \right], \nonumber 
\end{align}
where inequality $(a)$ holds because $\sum_j  x^{\epsilon}_{c,j}=1, \forall c$ and the action $j^*_c$ maximizes $V_t\hat r(c,j) - \sum_k Q^{(k)} W^{(k)}(c,j);$ and equality $(b)$ holds because of the definition of $X_{j}(t)$ in the Pessimistic-Optimistic Algorithm.
\end{proof}

Let $\mathbf x^{\epsilon_t}$ be a feasible solution to the tightened problem  \eqref{obj-fluid-tightened}-\eqref{resource limit-fluid-tightened} with $\epsilon=\epsilon_t$ if  $\epsilon_t\leq\delta;$ otherwise let $\mathbf x^{\epsilon_t}$ be a feasible solution to the tightened problem \eqref{obj-fluid-tightened}-\eqref{resource limit-fluid-tightened} with $\epsilon=\delta$.
From the lemma above, we can conclude
\begin{align}
& \mathbb E[L(t+1)-L(t)|\mathbf H(t)=\mathbf h] \nonumber\\
\leq& \mathbb E\left[\sum_{k}Q^{(k)}(t)
\left(\sum_jW^{(k)}(c(t),j) x^{\epsilon_t}_{c(t),j}+\epsilon_t\right)\Big|\mathbf H(t)=\mathbf h\right] \mathbb I(\epsilon_t \leq \delta) \nonumber\\
&+ \mathbb E\left[\sum_{k}Q^{(k)}(t)
\left(\sum_jW^{(k)}(c(t),j) x^{\epsilon_t}_{c(t),j}+\delta\right)\Big|\mathbf H(t)=\mathbf h\right] \mathbb I(\epsilon_t > \delta) \nonumber\\
&+ V_t \mathbb E\left[\sum_{j}\hat r(c(t),j) (X_{j}(t)-x^{\epsilon_t}_{c(t),j})\Big|\mathbf H(t)=\mathbf h\right] + K(1+\epsilon_t^2) + \sum_kQ^{(k)}(\epsilon_t -\delta)\mathbb I(\epsilon_t > \delta) \nonumber\\
=& \sum_{k}Q^{(k)}
\left(\sum_{j,c}p_cw^{(k)}(c,j) x^{\epsilon_t}_{c,j}+\epsilon_t\right) \mathbb I(\epsilon_t \leq \delta) \nonumber\\
&+ \sum_{k}Q^{(k)}
\left(\sum_{j,c}p_cw^{(k)}(c,j) x^{\epsilon_t}_{c,j} +\delta\right) \mathbb I(\epsilon_t > \delta) \nonumber\\
&+ V_t \mathbb E\left[\sum_{j}\hat r(c(t),j) (X_{j}(t)-x^{\epsilon_t}_{c(t),j})\Big|\mathbf H(t)=\mathbf h\right] + K(1+\epsilon_t^2) + \sum_kQ^{(k)}(\epsilon_t -\delta)\mathbb I(\epsilon_t > \delta), \nonumber
\end{align}
where the first inequality holds from Lemma \ref{lemma: drift-inequality} and 
the second equality holds because $W^{(k)}(c(t),j)$ are independent. 

From the definition of $\mathbf x^{\epsilon_t},$ we know that $\mathbf x^{\epsilon_t}$ is a feasible solution to \eqref{obj-fluid-tightened}-\eqref{resource limit-fluid-tightened} so both $\sum_{j,c}p_cw^{(k)}(c,j) x^{\epsilon_t}_{c,j} +\epsilon_t$ and $\sum_{j,c}p_cw^{(k)}(c,j) x^{\epsilon_t}_{c,j} +\delta$ are negative for any $k\in[K].$ Therefore, we have the following lemma.
\begin{lemma}\label{lemma: drift analysis}
Let $\mathbf x^{\epsilon_t}$ be a feasible solution to the tightened problem  \eqref{obj-fluid-tightened}-\eqref{resource limit-fluid-tightened} with $\epsilon=\epsilon_t$ if  $\epsilon_t\leq\delta;$ otherwise let $\mathbf x^{\epsilon_t}$ be a feasible solution to the tightened problem  \eqref{obj-fluid-tightened}-\eqref{resource limit-fluid-tightened} with $\epsilon=\delta$. The expected Lyapunov drift satisfies
\begin{align}
\mathbb E[L(t+1)-L(t)|\mathbf H(t)=\mathbf h] \leq& -  V_t \mathbb E\left[\sum_{j\in [J]} \hat r(c(t),j)  \left(x^{\epsilon_t}_{c(t),j}-X_{j}(t)\right) \Big| \mathbf H(t)=\mathbf h\right] \nonumber\\
&+ K(1+\epsilon_t^2) +\sum_k Q^{(k)} (\epsilon_t-\delta)\mathbb I(\epsilon_t>\delta)  . \nonumber
\end{align}   \hfill{$\square$}
\end{lemma}
Taking expectation with respect to $\mathbf H(t),$ dividing $V_t$ on both sides, and doing the telescope summation across rounds up to $\tau$ lead to 
\begin{align}
&\mathbb E\left[\sum_{t=1}^\tau \sum_{j} \hat r(c(t),j)  \left(x^{\epsilon_t}_{c(t),j}-X_{j}(t)\right) \right] \nonumber\\
&\leq \frac{\mathbb E[L(1)]}{V_1}-\frac{\mathbb E[L(\tau+1)]}{V_\tau} + \sum_{t=1}^\tau \frac{K(1+\epsilon^2_t)}{V_t} + \sum_{t=1}^\tau\frac{\sum_k Q^{(k)} (\epsilon_t-\delta)\mathbb I(\epsilon_t>\delta)}{V_t} \nonumber\\
&\leq \sum_{t=1}^\tau \frac{K(1+\epsilon^2_t)}{V_t}+\sum_{t=1}^\tau\frac{\sum_k Q^{(k)} (\epsilon_t-\delta)\mathbb I(\epsilon_t>\delta)}{V_t},\nonumber
\end{align}  where the last inequality holds because $L(0)=0$ and $L(\tau) \geq 0, \forall \tau \geq 0.$ The proof is completed by applying the following inequality 
\begin{align}
\mathbb E\left[\sum_k Q^{(k)}(t)\epsilon_t\mathbb I(\epsilon_t > \delta)\right] \leq Kt(\epsilon_t+\epsilon_t^2) \mathbb I(\epsilon_t > \delta). \nonumber
\end{align}
\section{Proof of Lemma \ref{lemma: elsilon gap}}
\label{app: epsilon gap}
Recall that $\mathbf x^*$ is the optimal solution to optimization problem \eqref{obj-fluid}-\eqref{resource limit-fluid}, so we have 
\begin{align*}
&\sum_{j\in [J]}  x_{c,j}^*=1, ~x_{c,j}^* \geq 0, \forall c \in \mathcal C, j\in [J],\\
&\sum_{c,j}p_cw^{(k)}(c,j) x_{c,j}^* \leq 0, ~\forall k \in \mathcal K. 
\end{align*}
Under Assumption \ref{assumption:slater}, there exists $\mathbf x^{\text{in}}$ such that \begin{align*}
&\sum_{j\in [J]}  x_{c,j}^{\text{in}}=1, ~x_{c,j}^* \geq 0, \forall c \in \mathcal C, j\in [J],\\
&\sum_{c,j}p_cw^{(k)}(c,j) x_{c,j}^{\text{in}} \leq -\delta, ~\forall k \in \mathcal K. 
\end{align*}
We now define $\mathbf x^{\epsilon_t} = \left(1-\frac{\epsilon_t}{\delta}\right) \mathbf x^* + \frac{\epsilon_t}{\delta} \mathbf x^{\text{in}}.$ We have 
\begin{align*}
&\sum_{j} x^{\epsilon_t}_{c,j} = \left(1-\frac{\epsilon_t}{\delta}\right) \sum_{j} x_{c,j}^* + \frac{\epsilon_t}{\delta}  \sum_{j} x_{c,j}^{\text{in}} = 1,
\end{align*}
and 
\begin{align*}
\sum_c p_c\sum_jw^{(k)}(c,j) x^{\epsilon_t}_{c,j} = \sum_c p_c\sum_jw^{(k)}(c,j) \left[\left(1-\frac{\epsilon_t}{\delta}\right)x_{c,j}^* +  \frac{\epsilon_t}{\delta} x_{c,j}^{\text{in}} \right] \leq -\epsilon_t, ~\forall k \in \mathcal K. 
\end{align*}
Therefore, $\mathbf x^{\epsilon_t}$ is a feasible solution to the tightened optimization problem \eqref{obj-fluid-tightened} - \eqref{resource limit-fluid-tightened} with $\epsilon=\epsilon_t.$ 

Recall that $\mathbf x^{\epsilon_t,*}$ is an optimal solution to the same tightened problem, so
we have
\begin{align*} 
\sum_{c\in\mathcal C,j\in [J]} p_c r(c,j) (x^*_{c,j} -x^{\epsilon_t,*}_{c,j}) 
\stackrel{(a)}{\leq} & \sum_{c,j}p_c r(c,j) \left(x^*_{c,j} - \left(1-\frac{\epsilon_t}{\delta}\right)  x^*_{c,j} - \frac{\epsilon_t}{\delta} x_{c,j}^{\text{in}}\right) \\
\stackrel{(b)}{\leq} & \sum_{c,j}p_c r(c,j) \left(x^*_{c,j} - \left(1-\frac{\epsilon_t}{\delta}\right)  x^*_{c,j}\right) \\
\leq & \sum_{c,j}p_c \left(x^*_{c,j}- \left(1-\frac{\epsilon_t}{\delta}\right)  x^*_{c,j} \right) \\
= & \frac{\epsilon_t}{\delta}\sum_{c,j}p_c x^*_{c,j} = \frac{\epsilon_t}{\delta}
\end{align*}
where (a) holds because $\mathbf x^{\epsilon_t,*}$ is the optimal solution and $\mathbf x^{\epsilon_t}$ is a feasible solution; (b) holds because $x_{c,j}^{\text{in}}\geq 0, \forall c,j.$

\section{Proof of Lemma \ref{lemma: bandits}}\label{app: linucb-r}
To prove Lemma \ref{lemma: bandits}, we first present two important results on ``self-normalized bound for vector-valued martingales'' and ``Confidence Ellipsoid"  from \cite{AbbPalSze_11}.
\begin{lemma}[Theorem 1 in \cite{AbbPalSze_11}]
Let $\{F_t\}_{t=0}^{\infty}$ be a filtration. Let $\{\eta_t\}_{t=1}^{\infty}$ be a real-valued stochastic process such that $\eta_t$ is $F_t$-measurable and $\eta_t$ is conditionally $1$-sub-Gaussian. Let $\{\phi_t\}_{t=1}^{\infty}$ be an $\mathbb R^d$ stochastic process such that $Y_{t-1}$ is $F_{t-1}$-measurable with $||\phi_t||\leq 1, \forall t.$ Assume $\Sigma_0$ is a $d \times d$ positive definite matrix. For any $t \geq 1,$ define
$$\Sigma_t = \Sigma_0 + \sum_{s=1}^{t} \phi_s \phi_s^{\dag} ~~~~~\text{and}~~~~~ S_t = \sum_{s=1}^{t} \eta_s \phi_s.$$ Then, for any $p > 0,$ with probability at least $1-p,$ for all $t \geq 1,$
$$||S_t||^2_{\Sigma_t^{-1}} \leq 2 \log \left(\frac{\det(\Sigma_t)^{1/2}}{p\det(\Sigma_0)^{1/2}}\right).$$ \label{lem: self-norm-bound}
\end{lemma}

\begin{lemma}[Theorem 2 in \cite{AbbPalSze_11}]\label{lem: confidence bound}
For any $p > 0,$ with probability at least $1-p,$  the following event $\mathcal E$ occurs
\begin{align}
||\hat \theta_{t-1} - \theta_{*}||_{\Sigma_{t-1}} \leq \sqrt{\beta_t(p)}, ~\forall t \geq 1, \nonumber
\end{align}
where $\sqrt{\beta_{\tau}(p)} = m+\sqrt{2\log \frac{1}{p} + d \log \left(1+\tau/d\right)}.$
\end{lemma}

Recall $\tilde r(c(t),j) = \max_{\theta \in \mathcal B_t}\langle \theta_t, \phi(c(t),j)\rangle = \langle \tilde{\theta}_t, \phi(c(t),j)\rangle$ and $\hat r(c(t),j) = \min\{1,\tilde r(c(t),j)\}.$ So when event $\mathcal E$ occurs, we have 
$$\tilde{r}(c(t),j) \geq \hat{r}(c(t),j)\geq r(c(t),j) ,$$ where the last inequality holds because $r(c,j)\in[0,1];$ and the following bound on $\hat r(c(t),j) - r(c(t),j), \forall j$:
\begin{align*}
\hat r(c(t),j) - r(c(t),j)
\stackrel{(a)}{\leq}& \tilde r(c(t),j) - r(c(t),j)\\
\leq& |\langle \tilde{\theta}_t - \theta_*, \phi(c(t),j)\rangle| \\
\stackrel{(b)}{\leq}& ||\tilde{\theta}_t - \theta_*||_{\Sigma_{t-1}} ||\phi(c(t),j)||_{\Sigma_{t-1}^{-1}} \\
\stackrel{(c)}{\leq}& ||\tilde{\theta}_t - \hat{\theta}_{t-1} ||_{\Sigma_{t-1}} ||\phi(c(t),j)||_{\Sigma_{t-1}^{-1}} + ||\hat{\theta}_{t-1} - \theta_* ||_{\Sigma_{t-1}} ||\phi(c(t),j)||_{\Sigma_{t-1}^{-1}} \\
\stackrel{(d)}{\leq}& \sqrt{\beta_t} ||\phi(c(t),j)||_{\Sigma_{t-1}^{-1}} + ||\hat{\theta}_{t-1} - \theta_* ||_{\Sigma_{t-1}} ||\phi(c(t),j)||_{\Sigma_{t-1}^{-1}} \\
\stackrel{(e)}{\leq}& 2\sqrt{\beta_t} ||\phi(c(t),j)||_{\Sigma_{t-1}^{-1}}\\
\stackrel{(f)}{\leq}& 2\sqrt{\beta_t} \min\left(1, ||\phi(c(t),j)||_{\Sigma_{t-1}^{-1}} \right)
\end{align*}
where (a) holds due to the definition of $\hat r(c(t),j)$ and $r(c(t),j)\leq1;$ (b) holds due to the Cauchy-Schwarz inequality; (c) holds due to the triangle inequality; (d) holds because $\tilde{\theta}_t \in \mathcal B_t = \{ \theta ~|~ ||\theta - {\hat \theta}_{t-1}||_{\Sigma_{t-1}} \leq \sqrt{\beta_t}\};$ (e) holds by Lemma \ref{lem: self-norm-bound} above; and (f) holds because $\hat r(c,j) \leq 1, \forall c, j$ according to the definition.    

Next, we introduce the Elliptical Potential
Lemma (Theorem 11.7 in \cite{CesLug_06}
and Theorem 19.4 in \cite{LatSze_20}) to bound $\sum_{t=1}^{\tau} \min\left(1, ||\phi(c(t),j)||_{\Sigma_{t-1}^{-1}} \right).$
\begin{lemma}
Let $\Sigma_0 =\mathbf I$ and $\phi_0, \phi_1, \cdots, \phi_{t-1} \in \mathbb R^d$ be a sequence of vectors with $||\phi_t|| \leq 1$ for any $t$ and $\Sigma_t =  \mathbf I + \sum_{s=1}^t \phi_s \phi_s^{\dag}.$ Then, 
$$\sum_{t=1}^{\tau} \min\left(1, ||\phi_s||_{\Sigma_{t-1}^{-1}}^2 \right) \leq 2\log \left(\frac{\det \Sigma_{\tau}}{\det \Sigma_{0}}\right) \leq 2d \log\left(\frac{ d + \tau}{ d }\right).$$ \label{lem: sum bound}
\end{lemma} 

According to Lemma \ref{lem: sum bound}, when event $\cal E$ occurs, we have
\begin{align*}
\sum_{t=1}^{\tau}  \hat r(c(t),A(t)) - r(c(t),A(t))  &\leq 2\sqrt{\tau\beta_{\tau}(p) \sum_{t=1}^{\tau} \min\left(1, ||\phi(c(t),A(t))||_{\Sigma_{t-1}^{-1}}^2 \right)} \\
&\leq \sqrt{8d\tau \beta_{\tau}(p) \log \left(\frac{d + \tau}{d}\right)}, 
\end{align*}
for any $j$ and $0<p<1$ with $\beta_{\tau}(p) = m + \sqrt{2 \log \frac{1}{p} + d\log \left(\frac{d+\tau}{d}\right)}.$

We are ready to prove Lemma \ref{lemma: bandits}. Let $p = 1/T$ and recall $\mathcal E$ in Lemma \ref{lem: confidence bound} occurs with probability $1-p$, i.e., $\mathbb P(\mathcal E) = 1-\mathbb P(\bar{\mathcal E})=1-1/T.$ So we have 
\begin{align*}
&\mathbb E\left[ \sum_{t=1}^{\tau} \left(\sum_{j}(\hat r(c(t),j)-r(c(t),j))X_{j}(t)\right)\right] \\
\leq& \tau \mathbb P(\bar{\mathcal E}) +  \mathbb E\left[ \sum_{t=1}^{\tau} \left| \hat r(c(t), A(t)) -  r(c(t), A(t))\right| \big| ~ \mathcal E\right],
\end{align*}
which is upper bounded by 
\begin{align*}
1 + \sqrt{8d\tau \beta_{\tau}(T^{-1}) \log \left(\frac{ d + \tau}{ d}\right)}.
\end{align*}
Furthermore, 
\begin{align*}
&\mathbb E\left[ \sum_{t=1}^{\tau}\sum_{j}(r(c(t),j)-\hat r(c(t),j))x^{\epsilon_t,*}_{c(t),j}\right] \\
\leq &  \tau \mathbb P(\bar{\mathcal E}) +\mathbb E\left[ \sum_{t=1}^{\tau}\sum_{j}(r(c(t),j)-\hat r(c(t),j))x^{\epsilon_t,*}_{c(t),j} \big|~ \mathcal E\right] \\
\leq & 1 
\end{align*}
where the last inequality holds because $\hat r(c(t),j)$ is an over-estimation of $r(c(t),j)$ under $\mathcal E$ in the algorithm, i.e., $\hat r(c(t),j) -  r(c(t),j) \geq 0, \forall c(t), j.$ 

\section{Proof of Lemma \ref{lemma: Q bound}}
\label{app: Q bound}
We first present a lemma which will be used to bound $\mathbb E[\sum_k Q^{(k)}(\tau+1)].$ The lemma is a minor variation of Lemma 4.1 \cite{Nee_16} (a similar result has also been established in an earlier paper \cite{Haj_82}). We present the proof for the completeness of the paper. 
\begin{lemma} \label{drif lemma}
Let $S(t)$ be a random process, $\Phi(t)$ be its Lyapunov function with $\Phi(0)=\Phi_0$ and $\Delta(t) = \Phi(t+1)-\Phi(t)$ be the Lyapunov drift. Given an increasing sequence $\{\varphi_t\},$ $\rho$ and $\nu_{\max}$ with $0 < \rho \leq \nu_{\max},$ if the expected drift $\mathbb E[\Delta(t) | S(t)=s]$ satisfies the following conditions:
\begin{itemize}
    \item[(i)] There exists constants $\rho > 0$ and $\varphi_t > 0$ such that $\mathbb E[\Delta(t)|S(t)=s] \leq -\rho$ when $\Phi(t) \geq \varphi_t,$ and

\item[(ii)] $|\Phi(t+1) - \Phi(t)| \leq \nu_{\max}$ holds with probability one;
\end{itemize}
then we have
\begin{align}
    \mathbb E[e^{\zeta \Phi(t)}] 
    \leq e^{\zeta\Phi_0} + \frac{2e^{\zeta(\nu_{\max}+\varphi_t)}}{\zeta\rho}, \label{eq: q exp bound}
\end{align}
where $\zeta = \frac{\rho}{\nu_{\max}^2+\nu_{\max}\rho/3}.$
\end{lemma}
\begin{proof}
The proof follows an induction argument. When $t=0$ and $\Phi(t)=\Phi_0,$ \eqref{eq: q exp bound} holds. Now suppose \eqref{eq: q exp bound} holds at slot $t.$ Then we study the upper bound on $E[e^{\zeta \Phi(t+1)}].$ Let $\zeta$ be a positive number satisfying $0 \leq \zeta\nu_{\max} \leq 3,$  from the proof of Theorem 8 in \cite{ChuLu_06}, we have for any $|x|\leq \nu_{\max,}$ $$e^{\zeta x} \leq 1 + \zeta x + \frac{(\zeta\nu_{\max})^2}{2(1-\zeta\nu_{\max}/3)}.$$
Recall $\Delta(t) = \Phi(t+1) - \Phi(t)$ and $|\Delta(t)| \leq \nu_{\max}$ for any $t,$ so we have
\begin{align*}
e^{\zeta \Phi(t+1)} &= e^{\zeta \Phi(t)} e^{\zeta \Delta(t)}\\
&\leq e^{\zeta \Phi(t)} \left(1+\zeta\Delta(t) + \frac{(\zeta\nu_{\max})^2}{2(1-\zeta\nu_{\max}/3)}\right)\\
&\leq e^{\zeta \Phi(t)} \left(1+\zeta\Delta(t) + \frac{\zeta\rho}{2}\right).
\end{align*}
Suppose $\Phi(t) > \varphi_t.$ From the inequality above and condition (i), we can obtain
$$\mathbb E[e^{\zeta \Phi(t+1)}|\Phi(t)] \leq \mathbb E\left[e^{\zeta \Phi(t)}\left(1+\zeta\Delta(t)+\frac{\zeta\rho}{2}\right)|\Phi(t)\right]\leq e^{\zeta \Phi(t)}\left(1-\frac{\zeta\rho}{2}\right).$$
Suppose $\Phi(t) \leq \varphi_t.$ From  condition (ii), we can obtain 
$$\mathbb E[e^{\zeta \Phi(t+1)}|\Phi(t)] \leq \mathbb E\left[e^{\zeta \Phi(t)} e^{\zeta \Delta(t)}|\Phi(t)\right]\leq e^{\zeta \Phi(t)}e^{\zeta\nu_{\max}}.$$
Combining these two cases, we obtain
\begin{align*}
\mathbb E[e^{\zeta \Phi(t+1)}] =& \mathbb E[e^{\zeta \Phi(t+1)} | \Phi(t) > \varphi_t] \mathbb P(\Phi(t) > \varphi_t)+\mathbb E[e^{\zeta \Phi(t+1)} | \Phi(t) \leq \varphi_t] \mathbb P(\Phi(t) \leq \varphi_t)\\
\leq & \left(1-\frac{\zeta \rho}{2}\right)\mathbb E[e^{\zeta \Phi(t)} | \Phi(t) > \varphi_t] \mathbb P(\Phi(t) > \varphi_t)+e^{\zeta \nu_{\max}}\mathbb E[e^{\zeta \Phi(t)} | \Phi(t) \leq \varphi_t] \mathbb P(\Phi(t) \leq \varphi_t)\\
=& \left(1-\frac{\zeta \rho}{2}\right)\mathbb E[e^{\zeta \Phi(t)}]+ \left(e^{\zeta\nu_{\max}}-\left(1-\frac{\zeta\rho}{2}\right)\right)\mathbb E[e^{\zeta \Phi(t)} | \Phi(t) \leq \varphi_t] \mathbb P(\Phi(t) \leq \varphi_t)\\
\leq& \left(1-\frac{\zeta\rho}{2}\right)\mathbb E[e^{\zeta \Phi(t)}]+ e^{\zeta(\nu_{\max}+\varphi_t)}. 
\end{align*}
Substituting \eqref{eq: q exp bound} for $t$ (the induction assumption) into the last inequality above, we have 
\begin{align*}
\mathbb E[e^{\zeta \Phi(t+1)}] 
\leq& \left(1-\frac{\zeta\rho}{2}\right)\left(e^{\zeta\Phi_0} + \frac{2e^{\zeta(\nu_{\max}+\varphi_t)}}{\zeta\rho}\right)+ e^{\zeta(\nu_{\max}+\varphi_t)} \\
\leq& \left(1-\frac{\zeta\rho}{2}\right)e^{\zeta\Phi_0} + \frac{2e^{\zeta(\nu_{\max}+\varphi_t)}}{\zeta\rho}\\\
\leq& e^{\zeta\Phi_0} + \frac{2e^{\zeta(\nu_{\max}+\varphi_{t+1})}}{\zeta\rho},
\end{align*}
which completes the proof.
\end{proof}

We now apply the lemma above to Lyapunov function \begin{align*}
\bar L(t) = \sqrt{\sum_{k} \left(Q^{(k)}(t)\right)^2} = ||\mathbf Q(t)||_2
\end{align*} as in \cite{ErySri_12}. We prove conditions (i) and (ii) in Lemma \ref{drif lemma} for $\bar L(t)$ are satisfied in the following subsection.

\subsection{Verifying Conditions (i) and (ii) for $\tilde{L}(t)$} 
Given $\mathbf H(t)=\mathbf h$ and $\bar L(t)\geq \varphi_t = \frac{4(V_t  + K(1+\epsilon_t^2))}{\delta},$ the conditional expected drift of $\bar L(t)$ is 
\begin{align*}
&\mathbb E[ ||\mathbf Q(t+1)||_2 - ||\mathbf Q(t)||_2 | \mathbf H(t) = \mathbf h] \\
=& \mathbb E\left[ \sqrt{||\mathbf Q(t+1)||_2^2} - \sqrt{||\mathbf Q(t)||_2^2} | \mathbf H(t) = \mathbf h\right] \\
\leq & \frac{1}{2||\mathbf Q||_2} \mathbb E[ ||\mathbf Q(t+1)||_2^2 - ||\mathbf Q(t)||_2^2 | \mathbf H(t) = \mathbf h] \\
\leq & -\frac{\delta}{2}\frac{||\mathbf Q||_1}{||\mathbf Q||_2} + \frac{V_t   + K(1+\epsilon_t^2)}{||\mathbf Q||_2} \\
\leq&  -\frac{\delta}{2} + \frac{V_t   + K(1+\epsilon_t^2)}{||\mathbf Q||_2} \\
\leq&  -\frac{\delta}{2} + \frac{V_t   + K(1+\epsilon_t^2)}{\varphi_t} = -\frac{\delta}{4}
\end{align*}
where the first inequality holds because $\sqrt{x}$ is a concave function; the second inequality holds by Lemma \ref{lemma: negative drift} below; the third inequality holds because $||\mathbf Q||_1 \geq ||\mathbf Q||_2;$ and the last inequality holds because $||\mathbf Q||_2 \geq \varphi_t.$ 
\begin{lemma}\label{lemma: negative drift}
Recall $L(t) = \frac{1}{2}\sum_{k} \left(Q^{(k)}(t)\right)^2 = ||\mathbf Q(t+1)||_2^2$ and $\epsilon_t \leq \delta/2$ for any $t.$ Under the Pessimistic-Optimistic algorithm, we have  
\begin{align*}
\mathbb E[L(t+1)-L(t)|\mathbf H(t)=\mathbf h] \leq-\frac{\delta}{2}\sum_{k}Q^{(k)}(t) + V_t + K(1+\epsilon_t^2). 
\end{align*}
\end{lemma}
\begin{proof}
According to the conditional expected drift in Lemma \ref{lemma: drift analysis} (given $\mathbf H(t)=\mathbf h$):
\begin{align}
& \mathbb E[L(t+1)-L(t)|\mathbf H(t)=\mathbf h] \nonumber\\
\leq& \sum_{k}Q^{(k)}\left(\sum_{c,j}p_cw^{(k)}(c,j) x^{\epsilon_t}_{c,j} +\epsilon_t\right) -  V_t\mathbb E\left[\sum_{j}\hat r(c(t),j)X_{j}(t)|\mathbf H(t)=\mathbf h\right] + K(1+\epsilon_t^2) \nonumber \\
\leq& \sum_{k}Q^{(k)}\left(\sum_{c,j}p_cw^{(k)}(c,j) x^{\epsilon_t}_{c,j} +\epsilon_t\right) +  V_t\mathbb E\left[\sum_{j}X_{j}(t)|\mathbf H(t)=\mathbf h\right] + K(1+\epsilon_t^2) \nonumber \\
\leq& \sum_{k}Q^{(k)}\left(\sum_{c,j}p_cw^{(k)}(c,j) x^{\epsilon_t}_{c,j} +\epsilon_t\right) + V_t + K(1+\epsilon_t^2) \nonumber \\
\leq& -\frac{\delta}{2}\sum_j\sum_{k}Q^{(k)} + V_t  + K(1+\epsilon_t^2) \label{cond:drift-exact}
\end{align}
where the second inequality holds because $|\hat r(c(t),j)| \leq 1,$ and the third inequality holds because rewards are $\sum_{j}x^{\epsilon_t}_{c,j}=\sum_{j}X_{j}(t)=1.$ The last inequality above holds because under Assumption \ref{assumption:slater}, there exists $\mathbf x$ satisfying $\sum_{j,c}p_cw^{(k)}(c,j) x_{c,j} \leq -\delta, \forall j, k.$ Therefore, by choosing $\mathbf x^{\epsilon_t}=\mathbf x,$ we obtain
$$\sum_{j,c}p_cw^{(k)}(c,j) x^{\epsilon_t}_{c,j} +\epsilon_t = \sum_{j,c}p_cw^{(k)}(c,j) x_{c,j} +\epsilon_t  \leq -\delta + \epsilon_t \leq -\frac{\delta}{2}, ~\forall k.$$

\end{proof}

Moreover, for condition (ii) in Lemma \ref{drif lemma}, we have 
\begin{align*}
||\mathbf Q(t+1)||_2 - ||\mathbf Q(t)||_2 \leq ||\mathbf Q(t+1) - \mathbf Q(t)||_2 \leq ||\mathbf Q(t+1) - \mathbf Q(t)||_1 \leq 2K,
\end{align*}
where the last inequality holds because $|Q^{(k)}(t+1)-Q^{(k)}(t)| \leq 1+\epsilon_t, \forall k$ based on Assumption \ref{assumption:constraints} and $\epsilon_t \leq \delta\leq 1, \forall t.$

\subsection{\textbf{Establishing a bound on $\mathbb E[\sum_{k} Q^{(k)}(t)]$}}\label{sec: bounds on queue}

Let $\varphi_t = \frac{4(V_t + K(1+\epsilon_t^2))}{\delta},$ $\rho = \frac{\delta}{4},$ and $\nu_{\max} = 2K.$ We apply Lemma \ref{drif lemma} for $\bar L(t)$ and obtain 
\begin{align*}
    \mathbb E\left[e^{\zeta ||\mathbf Q(t)||_2}\right] 
    \leq e^{\zeta||\mathbf Q(\tau')||_2} + \frac{2e^{\zeta(\nu_{\max}+\varphi_t)}}{\zeta\rho}\ \hbox{with} ~ \zeta = \frac{\rho}{\nu_{\max}^2 + \nu_{\max}\rho/3},
\end{align*}
which implies that
\begin{align}
    \mathbb E\left[e^{\frac{\zeta}{\sqrt{K}} ||\mathbf Q(t)||_1}\right] 
    \leq e^{\zeta||\mathbf Q(\tau')||_2} + \frac{2e^{\zeta(\nu_{\max}+\varphi_t)}}{\zeta\rho}, \label{eq:exp bound}
\end{align} because $||\mathbf Q(t)||_1 \leq \sqrt{K}||\mathbf Q(t)||_2.$
By Jensen's inequality, we have 
$$e^{\frac{\zeta}{\sqrt{K}} \mathbb E\left[||\mathbf Q(t)||_1\right]} \leq \mathbb E\left[e^{\frac{\zeta}{\sqrt{K}} ||\mathbf Q(t)||_1}\right] \leq e^{\zeta||\mathbf Q(\tau')||_2} + \frac{2e^{\zeta(\nu_{\max}+\varphi_t)}}{\zeta\rho},$$ which implies
\begin{align*}
    \mathbb E\left[\sum_{k} Q^{(k)}(t)\right]=& \mathbb E\left[||\mathbf Q(t)||_1\right]\\
    \leq& \frac{\sqrt{K}}{\zeta}\log\left(e^{\zeta||\mathbf Q(\tau')||_2} +\frac{2e^{\zeta(\nu_{\max}+\varphi_t)}}{\zeta\rho}\right) \\ 
    \leq& \frac{\sqrt{K}}{\zeta}\log\left(e^{\zeta||\mathbf Q(\tau')||_2}+\frac{8\nu_{\max}^2}{3\rho^2}e^{\zeta(\nu_{\max}+\varphi_t)}\right) \\
    \leq& \frac{\sqrt{K}}{\zeta}\log\left(\frac{11\nu_{\max}^2}{3\rho^2}e^{\zeta(\nu_{\max}+\varphi_t+||\mathbf Q(\tau')||_2)}\right)\\
    \leq& \frac{3\sqrt{K}\nu_{\max}^2}{\rho}\log\left(\frac{2\nu_{\max}}{\rho}\right)+\sqrt{K}\nu_{\max}+\sqrt{K}\varphi_t + \sqrt{K}||\mathbf Q(\tau')||_2\\
    =& \frac{3\sqrt{K}\nu_{\max}^2}{\rho}\log\left(\frac{2\nu_{\max}}{\rho}\right)+\sqrt{K}\nu_{\max}+\frac{4\sqrt{K}(V_t  + K(1+\epsilon_t^2))}{\delta} + \sqrt{K}||\mathbf Q(\tau')||_2
\end{align*}
where the second, third and fourth inequalities hold because $\zeta = \frac{\rho}{\nu_{\max}^2+\nu_{\max}\rho/3}$ and $0 < \rho \leq \nu_{\max}.$ The proof is completed by $||\mathbf Q(\tau')||_2 \leq \tau'+\sqrt{K}\sum_{t=1}^{\tau'}\epsilon_t.$

\section{Proof of Corollary \ref{cor: tail prob}} \label{app: tail prob}
Recall in Section \ref{sec: cv}, we have 
\begin{align}
\sum_{t=1}^{\tau} \sum_jW^{(k)}(c(t),j) X_{j}(t)\leq  Q^{(k)}(\tau+1)-\sum_{t=1}^{\tau}\epsilon_t. \nonumber
\end{align}
We analyze the tail probability in the following.  
\begin{align}
\Pr\left(Q^{(k)}(\tau+1)-\sum_{t=1}^{\tau}\epsilon_t\geq 0\right)
&\leq \Pr\left(||\mathbf Q(\tau + 1)||_1\geq\sum_{t=1}^{\tau}\epsilon_t \right) \nonumber\\
&\leq \frac{\mathbb E\left[e^{\frac{\zeta}{\sqrt{K}}||\mathbf Q(\tau + 1)||_1} \right]}{e^{\frac{\zeta}{\sqrt{K}}\sum_{t=1}^{\tau}\epsilon_t}}\nonumber\\
&\leq \frac{e^{\zeta||Q(\tau')||_2} + \frac{2e^{\zeta(\nu_{\max}+\varphi_{\tau+1})}}{\zeta\rho}}{e^{\frac{\zeta}{\sqrt{K}}(\sum_{t=1}^{\tau}\epsilon_t)}}\nonumber\\
&\leq \frac{11\nu_{\max}^2}{3\rho^2}e^{\zeta(\nu_{\max}+\varphi_{\tau+1}+||Q(\tau')||_2)-\frac{\zeta}{\sqrt{K}}(\sum_{t=1}^{\tau}\epsilon_t)}, \label{tail prob}
\end{align}
where the third inequality holds according to \eqref{eq:exp bound}.

Recall $\varphi_t = \frac{4(V_t + K(1+\epsilon_t^2))}{\delta},$ $V_t=\delta K^{0.25}\sqrt{\frac{2t}{3}}$ and $\epsilon_t=K^{0.75}\sqrt{\frac{6}{t}}.$ It is not hard to verify that for any $\tau > \tau'',$ we have
\begin{align}
\eqref{tail prob} \leq \frac{11\nu_{\max}^2}{3\rho^2} e^{\zeta\left(\frac{4K(1+\epsilon_{\tau+1}^2)}{\delta}+\tau'+\sqrt{K}\sum_{t=1}^{\tau'}\epsilon_{t}\right)} e^{-\frac{\zeta}{2}\sqrt{\frac{\tau}{K}}}\leq e^{-\frac{\zeta}{3}\sqrt{\frac{\tau}{K}}}, \label{eq:tail prob}
\end{align}
where 
\begin{align*}
\tau'' =&\left(\frac{200K^{2.5}}{\delta}\log\frac{27K}{\delta} +\frac{48K^{1.5}}{\delta}+\frac{288K^{2.5}}{\delta^2}\right)^2\\
\geq &\left(
\frac{6\sqrt{K}}{\zeta}\log \frac{11\nu_{\max}^2}{3\rho^2} + \frac{48K^{1.5}}{\delta}+\frac{288K^{2.5}}{\delta^2}\right)^2
\end{align*}
for $\zeta = \frac{\rho}{\nu_{\max}^2+\nu_{\max}\rho/3},$  $\rho = \frac{\delta}{4}$ and $\nu_{\max} = 2K,$ which completes the proof.  

\section{Linear Costs}
\label{app:linearcost}
In this section, we consider the case where the cost signals $W^{(k)}(c(t),A(t))$ are revealed after action $A(t)$ is taken and we do not have access to noisy estimates of the costs before the action is taken. However, we assume the costs can be linearly parameterized as in \cite{PacGhaBar_20}. Without loss of generality, we consider only single constraint, i.e., $K=1.$ 
After the learner takes action $A(t),$ beside receiving reward $R(c(t), A(t)),$ the learner also observes cost $$W(c(t), A(t))=w(c(t),A(t)) + \xi(t),$$ where $w(c, j) = \langle \mu_{*}, \psi(c, j) \rangle,$ $\psi(c,j)\in \mathbb R^d$ is a $d$-dimensional feature vector for (context, action) pair $(c,j),$ $\mu_{*} \in \mathbb R^d$ is an unknown underlying vector to be learned, and $\xi(t)$ is a zero-mean random variable. Here we assume the feature space of the costs also has dimension $d$ for convenience, but it is not necessary. 

Next, we present a standard assumption on the cost $w(c,j),$ a new version of the pessimistic-optimistic algorithm, and our main results on the regret and constraints violations. For convenience, we define  operator $(x)|^{h}_{l} = \max\{l,\min\{h, x\}\}, \forall l \leq h.$   

\begin{assumption}
The mean cost $w(c, j)=\langle \mu_*, \psi(c,j) \rangle\in [-1, 1]$ with $||\psi(c,j)||\leq 1, ||\mu_*|| \leq 1$ for any $c \in \mathcal C,\  j \in [J],$ and  $\xi(t)$ is zero-mean $1$-subgaussian conditioned on $\{\mathcal F_{t-1}, A(t)\}$. \label{assumption:W}
\end{assumption}

\vspace{0.1in}
\hrule
\vspace{0.1in}
\noindent{\bf A Pessimistic-Optimistic Algorithm}
\vspace{0.1in}
\hrule
\vspace{0.1in}

\noindent {\bf Initialization:} $Q(1)=0,$ $\mathcal B_1 = \{\theta | ||\theta||_{\Sigma_{r,0}} \leq \sqrt{\beta_1} \}, \mathcal U_1 = \{\mu | ||\mu||_{\Sigma_{w,0}} \leq \sqrt{\beta_1} \}, 
\Sigma_{r,0} = \Sigma_{w,0} = \mathbf I~\text{and}~\sqrt{\beta_1} = 1 + \sqrt{2\log T}.$

\noindent For $t=1,\cdots, T,$ 
\begin{itemize}[leftmargin=*]
\item \noindent {\bf Set:}  
$V_t=\frac{\delta d\sqrt{t}\log(1+T)}{4}$ and $\epsilon_t=\frac{4d\log(1+T)}{\sqrt{t}}.$

\item {\bf LinUCB (Optimistic):} Use LinUCB to estimate $r(c(t),j)$ and $W(c(t),j)$ for all $j:$ $$\hat r(c(t),j) = \tilde r(c(t),j)\big|^{1}_{-\infty} ~~\text{with}~~ \tilde r(c(t),j) = \max_{\theta \in \mathcal B_t} \langle \theta, \phi(c(t),j)\rangle.$$
$$\widecheck{W}(c(t),j) =  \widetilde W(c(t),j) \big|^{1}_{-1} ~~\text{with}~~ \widetilde W(c(t),j) = \min_{\mu \in \mathcal U_t} \langle \mu, \psi(c(t),j)\rangle.$$

\item {\bf MaxValue:} Compute {\em pseudo-action-value} of context $c(t)$ for all action $j,$ and take the action $j^*$ with the highest pseudo-action-value, breaking a tie arbitrarily
\begin{align*}
j^*\in \arg\max_j \underbrace{\hat r(c(t),j) - \frac{1}{V_t} \widecheck{W}(c(t),j)Q(t)}_{\text{pseudo action value of $(c(t),j)$}}.
\end{align*}

\item {\bf Dual Update (Pessimistic):} Update the estimate of dual variable $Q(t):$
\begin{align}
  Q(t+1) =& \left[Q(t) + \sum_j \widecheck  W(c(t),j)X_j(t) + \epsilon_t \right]^{+}.\label{eq:dual-update-W}
 \end{align}

\item {\bf Confidence Set Update:} Set $\sqrt{\beta_{t+1}} = 1 + \sqrt{2 \log T + d\log \left(\frac{d+t}{d}\right)}.$ Update $\Sigma_{r,t},$ $\Sigma_{w,t},$ $\hat \theta_t,$ $\hat \mu_t,$ $\mathcal B_{t+1},$ and $\mathcal U_{t+1}$ according to the received reward and cost signals $R(c(t), j^*)$ and $W(c(t), j^*):$
\begin{align*}
\Sigma_{r,t} = \Sigma_{r,t-1} + \phi(c(s),j^*) \phi^{\dag}(c(s),j^*), ~~~~
\hat \theta_t = \Sigma^{-1}_{r,t}\sum_{s=1}^t \phi(c(s),A(s)) R(c(s), A(s)), \\
\Sigma_{w,t} = \Sigma_{w,t-1} + \psi(c(s),j^*) \psi^{\dag}(c(s),j^*), ~~~~
\hat \mu_t = \Sigma^{-1}_{w,t}\sum_{s=1}^t \psi(c(s),A(s)) W(c(s), A(s)), \\
\mathcal B_{t+1} = \{\theta ~|~ ||\theta - \hat \theta_{t} ||_{\Sigma_{r,t}} \leq  \sqrt{\beta_{t+1}} \},~~~~ \mathcal U_{t+1} = \left\{\mu ~|~ ||\mu - \hat \mu_{t} ||_{\Sigma_{w,t}} \leq \sqrt{\beta_{t+1}} \right\}.
\end{align*}
\end{itemize}
\vspace{0.1in}
\hrule
\vspace{0.1in}
 
\begin{theorem}[Formal Statement of Theorem \ref{thm:linearcost}]
Under Assumptions \ref{assumption:reward}-\ref{assumption:W}, the pessimistic-optimistic algorithm presented in this Section
achieves the following regret and constraint violations bounds for any $\tau \in [T]:$ 
\begin{align*}
\mathcal R(\tau) \leq& \frac{235d^4\log^4(1+T)}{\delta^3}+ \frac{(8d+24)\sqrt{\tau}\log(1+T)}{\delta}+2+10d\sqrt{\tau}\log(1+T)
\\
\cal V(\tau)\leq&  \left(\frac{128d^2\log^2(1+T)}{\delta^2}+\frac{48}{\delta}\log\left(\frac{16}{\delta}\right)+\frac{12}{\delta} + 2d(4-\sqrt{\tau})\log(1+T) \right)^{+}.
\end{align*} \hfill{$\square$}
\label{thm:formal-UCB-W}
\end{theorem}

According to Theorem \ref{thm:formal-UCB-W}, we have $${\cal R}(\tau)=\tilde{\cal O}\left(\frac{d}{\delta}\sqrt{\tau}+\frac{d^4}{\delta^3}\right).$$ We observe that the regret grows sub-linearly in round $\tau$ and polynomially in the dimension of the reward and cost features $d,$ and the inverse of Slater's constant $\delta.$ 
For the constraint violation, we observe 
\begin{equation}
{\cal V}(\tau)=\begin{cases}
O\left(\frac{d^2\log^2(1+T)}{\delta^2} \right)& {\tau} \leq \left(4  + 30\log\frac{16}{\delta}+ \frac{64d\log T}{\delta^2}\right)^2\\
0&\hbox{otherwise}
\end{cases}.
\end{equation}
We observe that the constraint violation requires $O\left(\frac{d^2\log^2T}{\delta^4}\right)$ rounds to reach zero because it takes time to learn the cost parameter vector $\mu_*$.

The proof of Theorem \ref{thm:formal-UCB-W} follows that of Theorem \ref{thm:formal-UCB}. Lemmas \ref{lemma: drift analysis-W} and \ref{lemma: negative drift-W} corresponds to Lemmas \ref{lemma: drift analysis} and \ref{lemma: negative drift}, respectively, which include  additional terms due to estimating $\widecheck{W}(c, j).$ To proceed, we first provide a confidence bound on true cost parameters of $\mu_{*}$ similar with Lemma \ref{lem: confidence bound}.

\begin{lemma}[Confidence Bound of $\mu_{*}$]\label{lem: confidence bound-W}
For any $p > 0,$ with probability at least $1-p,$ for all $t \geq 1,$
\begin{align}
||\hat \mu_{t-1} - \mu_{*}||_{\Sigma_{w,t-1}} \leq \sqrt{\beta_t(p)}, \nonumber
\end{align}
where $\sqrt{\beta_{\tau}(p)} = 1+\sqrt{2\log (1/p) + d \log \left(1+\tau/d\right)}.$
\end{lemma}

Based on Lemma \ref{lem: confidence bound-W} above and the definition of $\widecheck{W}(c(t), j),$ we have for all $t \geq 1,$
\begin{align}\label{eq: key fact W}
    \mathbb P\left(\widecheck{W}(c(t), j) - w(c(t),j) \leq 0\right) \geq 1- p, ~~\forall p > 0.
\end{align}
We next establish the regret and constraints violations of the pessimistic-optimistic algorithm to prove Theorem \ref{thm:formal-UCB-W}. 
\subsection{Regret Bound}
Define $L(t) = \frac{1}{2} \left(Q(t)\right)^2.$ We first study the expected drift conditioned on the current state $\mathbf H(t)=[\mathbf Q(t),  \hat{\mathbf r}(t), \widecheck{\mathbf W}(t)]=\mathbf h=[\mathbf Q, \hat{\mathbf r},\widecheck{\mathbf W}]$ as in Section \ref{sec:main proof}. 
\begin{lemma}\label{lemma: drift analysis-W}
Let $\mathbf x^{\epsilon_t}$ be a feasible solution to the tightened problem  \eqref{obj-fluid-tightened}-\eqref{resource limit-fluid-tightened} with $\epsilon=\epsilon_t$ if  $\epsilon_t\leq\delta;$ otherwise $\mathbf x^{\epsilon_t}$ be a feasible solution to the tightened problem  \eqref{obj-fluid-tightened}-\eqref{resource limit-fluid-tightened} with $\epsilon=\delta$. The expected Lyapunov drift satisfies
\begin{align}
\mathbb E[L(t+1)-L(t)|\mathbf H(t)=\mathbf h] \leq& -  V_t \mathbb E\left[\sum_{j\in [J]} \hat r(c(t),j)  \left(x^{\epsilon_t}_{c(t),j}-X_{j}(t)\right) \Big| \mathbf H(t)=\mathbf h\right] \\
&+ Q(\epsilon_t-\delta)\mathbb I(\epsilon_t > \delta)  + (2+\epsilon_t+\epsilon_t^2). \nonumber
\end{align}  
\end{lemma}
\begin{proof}
By following the same steps in the proof of Lemma \ref{lemma: drift in regret} (or Lemma \ref{lemma: drift analysis}), we can conclude
\begin{align}
& \mathbb E[L(t+1)-L(t)|\mathbf H(t)=\mathbf h] \nonumber\\
\leq& \mathbb E\left[Q
\left(\sum_{j}\widecheck W(c(t),j) x^{\epsilon_t}_{c(t),j}+\epsilon_t\right)\Big|\mathbf H(t)=\mathbf h\right] \mathbb I(\epsilon_t \leq \delta) \nonumber\\
+& \mathbb E\left[Q
\left(\sum_{j}\widecheck W(c(t),j) x^{\epsilon_t}_{c(t),j}+\delta\right)\Big|\mathbf H(t)=\mathbf h\right] \mathbb I(\epsilon_t > \delta) \nonumber\\
&+ V_t \mathbb E\left[\sum_{j}\hat r(c(t),j) (X_{j}(t)-x^{\epsilon_t}_{c(t),j})\Big|\mathbf H(t)=\mathbf h\right] + (1+\epsilon_t^2) + Q(\epsilon_t -\delta)\mathbb I(\epsilon_t > \delta) \label{eq: drift-term-middle-W}.
\end{align}
Next, we connect $\widecheck W(c(t),j)$ with $w(c(t),j)$ in \eqref{eq: drift-term-middle-W} as follows
\begin{align}
\eqref{eq: drift-term-middle-W}\stackrel{(a)}{=}& Q
\left(\sum_{c,j}p_cw(c,j) x^{\epsilon_t}_{c,j}+\epsilon_t\right)\mathbb I(\epsilon_t \leq \delta) + Q
\left(\sum_{c,j}p_cw(c,j) x^{\epsilon_t}_{c,j}+\delta\right)\mathbb I(\epsilon_t > \delta) \nonumber\\
&+ Q\mathbb E\left[\sum_{j}
\left(\widecheck W(c(t),j) - w(c(t),j)\right)x^{\epsilon_t}_{c(t),j}\right] -  V_t \mathbb E\left[\sum_{j} \hat r(c(t),j)  (x^{\epsilon_t}_{c(t),j}-X_{j}(t)) \Big| \mathbf H(t)=\mathbf h\right] \nonumber\\
&+ (1+\epsilon_t^2) + Q(\epsilon_t-\delta)\mathbb I(\epsilon_t>\delta),\nonumber\\
\stackrel{(b)}{\leq}& Q\mathbb E\left[\sum_{j}
\left(\widecheck W(c(t),j) - w(c(t),j)\right)x^{\epsilon_t}_{c(t),j}\right]-  V_t \mathbb E\left[\sum_{j} \hat r(c(t),j)  (x^{\epsilon_t}_{c(t),j}-X_{j}(t)) \Big| \mathbf H(t)=\mathbf h\right] \nonumber\\
& + (1+\epsilon_t^2)+Q(\epsilon_t-\delta)\mathbb I(\epsilon_t>\delta),\nonumber\\
\stackrel{(c)}{\leq}&  -  V_t \mathbb E\left[\sum_{j} \hat r(c(t),j)  (x^{\epsilon_t}_{c(t),j}-X_{j}(t)) \Big| \mathbf H(t)=\mathbf h\right] +Q(\epsilon_t-\delta)\mathbb I(\epsilon_t>\delta)+ (2+\epsilon_t+\epsilon_t^2),\nonumber
\end{align}  
where (a) holds by adding and subtracting $w(c,j);$ (b) holds because $\mathbf x^{\epsilon_t}$ is a feasible solution to \eqref{obj-fluid-tightened}-\eqref{resource limit-fluid-tightened}; and (c) holds by invoking $p=1/T$ in \eqref{eq: key fact W} and noting $Q(t) \leq T(1+\epsilon_t), \forall t,$ that
\begin{align*}
&Q\mathbb E\left[
\left(\widecheck W(c(t),j) - w(c(t),j)\right)x^{\epsilon_t}_{c(t),j}\right] \\
\leq& Q\mathbb E\left[\left.
\left(\widecheck W(c(t),j) - w(c(t),j)\right)x^{\epsilon_t}_{c(t),j} ~\right|~ \widecheck W(c(t),j) > w(c(t),j)\right]\Pr\left(\widecheck W(c(t),j) > w(c(t),j)\right) \\
\leq& \frac{Q}{T}\leq 1+\epsilon_t.
\end{align*}
\end{proof}

Following the same analysis in Section \ref{sec: regret}, we have the additional term 
\begin{align}
\mathbb E\left[\sum_{t=1}^{\tau} Q(t)\epsilon_t\mathbb I(\epsilon_t > \delta)\right] \leq \sum_{t=1}^{\tau} t(\epsilon_t+\epsilon_t^2) \mathbb I(\epsilon_t > \delta) \nonumber
\end{align}
in the regret in Section \ref{sec: regret}, and we conclude  
\begin{align*}
{\cal R}(\tau)\leq & \sum_{t=1}^{\tau}\frac{ t(\epsilon_t+\epsilon_t^2) \mathbb I(\epsilon_t > \delta)}{V_t}+\sum_{t=1}^\tau \frac{\epsilon_t}{\delta}+ \sum_{t=1}^{\tau}\frac{2+\epsilon_t+\epsilon_t^2}{V_t} + 2 + \sqrt{8d\tau \beta_{\tau}(T^{-1}) \log \left(\frac{ d + \tau}{ d}\right)}.
\end{align*}
By choosing $\epsilon_t=\frac{4d\log(1+T)}{\sqrt{t}}$ and $V_t=\frac{\delta d\sqrt{t}\log(1+T)}{4},$ we obtain
\begin{align*}
\mathcal R(\tau) \leq& \frac{235d^4\log^4(1+T)}{\delta^3}+  \frac{(8d+24)\sqrt{\tau}\log(1+T)}{\delta}+2+10d\sqrt{\tau}\log(1+T).
\end{align*}

\subsection{Constraints Violations}\label{sec: cv-W}
\begin{lemma}[A new version of Lemma \ref{lemma: negative drift}]\label{lemma: negative drift-W}
Assume $\epsilon_t \leq \delta/2.$ Under the Pessimistic-Optimistic algorithm, we have  
\begin{align*}
\mathbb E[L(t+1)-L(t)|\mathbf H(t)=\mathbf h] \leq-\frac{\delta}{2}Q(t) + V_t + (2+\epsilon_t+\epsilon_t^2). 
\end{align*}
\end{lemma}

Define $\tau'$ be the first time such that $\epsilon_{\tau'} \leq \delta/2,$ that is, $\epsilon_\tau > \delta/2, \forall \tau < \tau'.$ Note that $Q(\tau') \leq \sum_{t=1}^{\tau'}(1+\epsilon_t)$  because $Q$ can increase by at most $(1+\epsilon_t)$ in each round. 
\begin{lemma}\label{lemma: Q bound-W} 
For any time $\tau\in [T]$ such that $\tau \geq \tau',$ i.e., $\epsilon_\tau \leq \delta/2,$ we have
\begin{align*}
\mathbb E\left[Q(\tau)\right]\leq \frac{48}{\delta}\log\left(\frac{16}{\delta}\right)+2+\frac{4(V_\tau + (2+\epsilon_t+\epsilon_t^2))}{\delta} + \sum_{t=1}^{\tau'}(1+\epsilon_t).
\end{align*}
\end{lemma}

According to the dynamic defined in \eqref{eq:dual-update-W}, we have
\begin{align*}
&\sum_{t=1}^\tau \sum_jW(c(t),j) X_{j}(t) \\
\leq& Q(\tau+1) + \sum_{t=1}^\tau \left(W(c(t),j) - \widecheck W(c(t),j)\right) X_{j}(t) -\sum_{t=1}^\tau\epsilon_t\\
\leq& Q(\tau+1) + \sum_{t=1}^\tau \left|W(c(t),j) - \widetilde W(c(t),j)\right| X_{j}(t) -\sum_{t=1}^\tau\epsilon_t,
\end{align*} where we used the fact $Q(0)=0.$ 
Following the steps in the proof of Lemma \ref{lemma: bandits} in Appendix \ref{app: linucb-r}, we have
\begin{align*}
\mathbb E\left[ \sum_{t=1}^{\tau} \left|W(c(t),j) - \widetilde W(c(t),j)\right|X_{j}(t)\right] 
\leq& 1 + \sqrt{8d\tau \beta_{\tau}(T^{-1}) \log \left(\frac{ d + \tau}{ d}\right)} \\
\leq& 1+5d\sqrt{\tau}\log(1+T),
\end{align*}
where we used $m= 1$ in Assumption \ref{assumption:W}. 

We next follow the steps in Section \ref{sec: cv}. By choosing $\epsilon_t=\frac{4d\log(1+T)}{\sqrt{t}}$ and $V_t=\frac{\delta d\sqrt{t}\log(1+T)}{4},$  we have $\epsilon_\tau \leq \delta/2$ for $\tau \geq \frac{64d^2\log^2(1+T)}{\delta^2}.$ From Lemma \ref{lemma: Q bound-W}, we conclude that 
\begin{align*}
\mathcal V(\tau) \leq& \left(\mathbb E[Q(\tau+1)] + 1 + 5d\sqrt{\tau}\log(1+T)- \sum_{t=1}^{\tau} \epsilon_t\right)^{+}\\
\leq& \left(\frac{128d^2\log^2(1+T)}{\delta^2}+\frac{48}{\delta}\log\left(\frac{16}{\delta}\right)+\frac{12}{\delta} + 2d(4-\sqrt{\tau})\log(1+T) \right)^{+}. 
\end{align*}

\end{document}